\documentclass[a4paper, 11pt]{article}
\usepackage{jmlr2e}

\usepackage{mathptmx}       % selects Times Roman as basic font
\usepackage{helvet}         % selects Helvetica as sans-serif font
\usepackage{courier}        % selects Courier as typewriter font
\usepackage{type1cm}        % activate if the above 3 fonts are
\usepackage{makeidx}         % allows index generation
\usepackage{graphicx}        % standard LaTeX graphics tool
\usepackage{multicol}        % used for the two-column index
\usepackage[bottom]{footmisc}% places footnotes at page bottom
\usepackage{natbib}
\usepackage{amsmath}
\usepackage{amsfonts}
\usepackage{stmaryrd}

\usepackage{geometry}
\newtheorem{prop}{Proposition}
\usepackage{subcaption}
\usepackage{booktabs}

\newcommand\blfootnote[1]{%
  \begingroup
  \renewcommand\thefootnote{}\footnote{#1}%
  \addtocounter{footnote}{-1}%
  \endgroup
}
\usepackage{tikz}
\usetikzlibrary{arrows,calc,positioning,decorations.pathreplacing} 
\tikzset{
    >=stealth',
    punkt/.style={
           circle,
           draw=black, thick,
           minimum height=1.75em,
           inner sep=0pt,
           text centered},
    pil/.style={
           ->,
           thick}
}

\newcommand{\Pa}[1]{\text{Pa}({#1}; \mathcal{G})}
\newcommand{\Pahat}[1]{\text{Pa}({#1}; \hat{\mathcal{G}})}
\newcommand{\Ne}[1]{\text{Ne}({#1}; \mathcal{S})}

\makeatletter
\newcommand*{\indep}{%
  \mathbin{%
    \mathpalette{\@indep}{}%
  }%
}
\newcommand*{\nindep}{%
  \mathbin{%                   % The final symbol is a binary math operator
    \mathpalette{\@indep}{\not}% \mathpalette helps for the adaptation
  }%
}
\def\layersep{.38cm}
\def\inlsep{.4}
\newcommand*{\@indep}[2]{%
  \sbox0{$#1\perp\m@th$}%        box 0 contains \perp symbol
  \sbox2{$#1=$}%                 box 2 for the height of =
  \sbox4{$#1\vcenter{}$}%        box 4 for the height of the math axis
  \rlap{\copy0}%                 first \perp
  \dimen@=\dimexpr\ht2-\ht4-.2pt\relax
  \kern\dimen@
  {#2}%
  \kern\dimen@
  \copy0 %                       second \perp
} 
\makeatother

\newcommand{\CG}{CGNN}
\def\X{\mbox{$\mathbf{X}$}}

\makeindex             % used for the subject index
\begin{document}

\title{Learning Functional Causal Models with Generative Neural Networks \thanks{$^*$ This article is a preprint of the chapter with the same name in the book \textit{Explainable and Interpretable Models in Computer Vision and Machine Learning}. Springer Series on Challenges in Machine Learning. 2018. Cham: Springer International Publishing. Editors: Hugo Jair Escalante, Sergio Escalera, Isabelle Guyon, Xavier Bar\'o and Ya\c{g}mur G{\"u}{\c{c}}l{\"u}t{\"u}rk. https://doi.org/10.1007/978-3-319-98131-4}}

%and Bar{\'o}, X and , Y and G{\"u}{\c{c}}l{\"u}, U and van Gerven.  }

%. Springer Series on Challenges in Machine Learning. 2018. Cham: Springer International Publishing. Editors:Escalante, HJ and Escalera, S and Guyon, I and Bar{\'o}, X and G{\"u}{\c{c}}l{\"u}t{\"u}rk, Y and G{\"u}{\c{c}}l{\"u}, U and van Gerven. DOI : \url{https://doi.org/10.1007/978-3-319-98131-4.}

% }.}

% \author{Olivier Goudet\thanks{Joint first author. Rest of authors ordered alphabetically.} \and Diviyan Kalainathan$^{\ast}$ \and Philippe Caillou, Isabelle Guyon \and David Lopez-Paz \and Mich\`ele Sebag}

% \institute{Olivier Goudet  \and Diviyan Kalainathan \and Philippe Caillou \and Isabelle Guyon  \and Mich\`ele Sebag \at
% Team TAU - CNRS, INRIA, Universit\'e Paris Sud,  Universit\'e Paris Saclay - France \and
% David Lopez-Paz \at Facebook AI Research \and
% \email{olivier.goudet@inria.frolivier.goudet@inria.fr, Diviyan.kalainathan@lri.fr, Caillou@lri.fr, guyon@chalearn.org, david@lopezpaz.org, sebag@lri.fr}}
% \title{Learning with Mixtures of Trees}

\author{\name Olivier Goudet $^{\dagger}$ \email olivier.goudet@inria.fr \\
       \addr TAU, LRI, CNRS, INRIA, Universit\'e Paris-Sud\\
       Gif-Sur-Yvette, 91190, France
       \AND
       \name Diviyan Kalainathan$^{\dagger}$ \email diviyan.kalainathan@inria.fr \\
       \addr  TAU, LRI, CNRS, INRIA, Universit\'e Paris-Sud\\
       Gif-Sur-Yvette, 91190, France
     \AND
    \name Philippe Caillou \email Caillou@lri.fr \\
       \addr  TAU, LRI, CNRS, INRIA, Universit\'e Paris-Sud\\
       Gif-Sur-Yvette, 91190, France
     \AND
     \name Isabelle Guyon \email isabelle.guyon@u-psud.fr \\
       \addr  TAU, LRI, CNRS, INRIA, Universit\'e Paris-Sud\\
       Gif-Sur-Yvette, 91190, France
     \AND
     \name David Lopez-Paz \email dlp@fb.fr \\
       \addr  Facebook AI Research\\
       Paris, 75002, France
     \AND
     \name Mich\`ele Sebag \email michele.sebag@lri.fr \\
       \addr  TAU, LRI, CNRS, INRIA, Universit\'e Paris-Sud\\
       Gif-Sur-Yvette, 91190, France
     }

% \editor{Leslie Pack Kaelbling}

\maketitle

%\institute{Name of First Author \at Name, Address of Institute, \email{name@email.address}
%\and Name of Second Author \at Name, Address of Institute \email{name@email.address}}

%\maketitle

\noindent

\abstract{ We introduce a new approach to functional causal modeling from observational data, called {\em Causal Generative Neural Networks} (\CG). \CG\ leverages the power of neural networks to learn a generative model of the joint distribution of the observed variables, by minimizing the Maximum Mean Discrepancy between generated and observed data. An approximate learning criterion is proposed to scale the computational cost of the approach to linear complexity in the number of observations.
The performance of \CG\ is studied throughout three experiments.
Firstly, \CG\ is applied to cause-effect inference, where the task is to identify the best causal hypothesis out of ``$X\rightarrow Y$'' and ``$Y\rightarrow X$''. 
Secondly, \CG\ is applied to the problem of identifying v-structures and  conditional independences. Thirdly, \CG\ is applied to multivariate functional causal modeling: given a skeleton describing the direct dependences in a set of random variables $\textbf{X} = [X_1, \ldots,
X_d]$, \CG\ orients the edges in the skeleton to uncover the directed acyclic causal graph describing the causal structure of the random variables.
On all three tasks, \CG\ is extensively assessed on both artificial and real-world data, comparing favorably to the state-of-the-art. Finally, \CG\ is extended to handle the case of confounders, where latent variables are involved in the overall causal model.}
\blfootnote{$^{\dagger}$ First joint author. Rest of authors ordered alphabetically.}

\section{Introduction }

Deep learning models have shown extraordinary predictive abilities, breaking
records in image classification \citep{krizhevsky2012imagenet}, speech
recognition \citep{hinton2012deep}, language translation
\citep{cho2014learning}, and reinforcement learning
\citep{silver2016mastering}.  However, the predictive focus of black-box deep
learning models leaves little room for explanatory power.  More generally, current machine learning paradigms offer no protection to avoid mistaking
correlation by causation.  For example, consider the prediction of target variable $Y$ given features $X$ and $Z$, assuming
that the underlying generative process is described by the
equations:
% Simplification: no indices when can be avoided.
\begin{align*}
  X, E_Y, E_Z &\sim \text{Uniform}(0, 1),\\
  Y           &\leftarrow 0.5 X + E_Y,\\
  Z         &\leftarrow Y + E_Z,
\end{align*}

with $(E_Y, E_Z)$ additive noise variables. The above model states that the values of $Y$ are computed as a function of the values of $X$ (we say that $X$ causes $Y$), and that the
values of $Z$ are computed as a function of the values of $Y$ ($Y$ causes $Z$).  The
``assignment arrows'' emphasize the asymmetric relations between all three random variables.  However, as $Z$ provides a stronger signal-to-noise ratio than $X$ for the prediction of $Y$, the best regression solution in terms of least-square error is 
$$\hat{Y} = 0.25 X + 0.5 Z$$
The above regression model, a typical case of inverse regression
after \cite{goldberger1984reverse}, would
wrongly explain some changes in $Y$ as a function of changes in $Z$, although $Z$ does not cause $Y$. In this simple case, there exists approaches overcoming the inverse regression mistake and uncovering all true cause-effect
relations \citep{hoyer2009nonlinear}. In the general case however, mainstream machine learning approaches fail to understand the relationships between all three distributions, and might attribute some effects on $Y$ to changes in $Z$. 

Mistaking correlation for causation can be catastrophic for agents who must plan, reason, and decide based on observations.  Thus, discovering causal structures is 
of crucial importance.

The gold standard to discover causal relations is to perform experiments
\citep{pearl2003causality}. However, experiments are in many cases expensive,
unethical, or impossible to realize. In these situations, there is a
need for \emph{observational causal discovery}, that is, the estimation of
causal relations from observations alone \citep{spirtes2000causation,
PetJanSch17}. 

In the considered setting, {\em observational} empirical data (drawn independent and identically distributed from an unknown distribution) is given as a set of $n$ samples of real valued feature vectors of dimension $d$. We denote the corresponding random vector as $\mathbf{X} = \left[X_1,..., X_d\right] $. 
We seek a Functional Causal Model (FCM), also known as Structural Equation Model (SEM), that best matches the underlying data-generating mechanism(s) in the following sense: under relevant manipulations/interventions/experiments the FCM would produce data distributed similarly to the real data obtained in similar conditions. 

Let intervention \textit{do(X=x)} be defined as the operation on distribution obtained by clamping variable $X$ to value $x$, while the rest of the system remains unchanged \citep{pearl2009causality}. It is said that variable $X_i$ is a \textbf{direct cause} of $X_j$ with respect to $X_1,..., X_d$ iff different interventions on variable $X$ result in different marginal distributions on $X_j$, everything else being equal:
\begin{equation}
P_{X_j | \text{do}(X_i=x,\textbf{X}_{\backslash ij}=\textbf{c})} \neq P_{X_j | \text{do}(X_i=x',\textbf{X}_{\backslash ij}=\textbf{c})}
\end{equation}
with $\textbf{X}_{\backslash ij} :=X_{\{1,...,d\} \backslash i,j}$ the set of all variables except $X_i$ and $X_j$, scalar values 
$x \neq x'$, and vector value $\textbf{c}$.  Distribution $P_{X_j | \text{do}(X_i=x,\textbf{X}_{\backslash ij}=c)}$ is the resulting interventional distribution of the variable $X_j$ when the variable $X_i$ is clamped to value $x$, while keeping all other variables at a fixed value \citep{mooij2016distinguishing}. 

As said, conducting such interventions to determine direct causes and effects raises some limitations. For this reason, this paper focuses on learning the causal structure from observational data only, where the goal and validation of the proposed approach is to match the known ``ground truth'' model structure.

A contribution of the paper is to unify several state-of-art methods into one single consistent and more powerful approach. On the one hand, leading researchers at UCLA, Carnegie Mellon, University of Crete
 and elsewhere have developed powerful algorithms exploiting Markov properties of directed acyclic graphs (DAGs). \citep{spirtes2000causation,tsamardinos2006max,pearl2009causality} On the other hand, the T\"ubingen School has proposed new and powerful functional causal models (FCM) algorithms exploiting the asymmetries in the joint distribution of cause-effect pairs. \citep{hoyer2009nonlinear,stegle2010probabilistic,daniusis2012inferring,mooij2016distinguishing} 

In this paper, the learning of functional causal models is tackled in the search space of generative neural networks \citep{kingma2013auto,goodfellow2014generative}, and aims at the functional causal model (structure and parameters), best fitting the underlying data generative process. The merits of the proposed approach, called Causal Generative Neural Network (\CG) are extensively and empirically demonstrated compared to the state of the art on artificial and real-world benchmarks.

This paper is organized as follows: Section 2 introduces the problem of learning an FCM and the underlying assumptions. Section 3 briefly reviews and discusses the state of the art in causal modeling. The FCM modeling framework within the search space of generative neural networks is presented in Section 4. Section 5 reports on an extensive experimental validation of the approach comparatively to the state of the art for pairwise cause-effect inference and graph recovery. An extension of the proposed framework to deal with potential confounding variables is presented in Section \ref{confounder}. The paper concludes in Section \ref{discussion} with some perspectives for future works.

\section{Problem setting \label{ProblemSettings}}\label{sec:setting}

A Functional Causal Model (FCM) upon a random variable vector $\textbf{X} = [X_1, \ldots, X_d]$ is a triplet $(\mathcal{G}, f, \mathcal{E})$, representing a set of equations: 
\begin{equation}   {X}_i  \leftarrow {f}_i({X}_{\Pa{i}}, {E}_i), {E}_i \sim {\mathcal{E}}, \mbox{ for } i=1,\ldots, d
    \label{eq:1}
\end{equation}

Each equation characterizes the direct
causal relation explaining variable $X_i$ from the set of its causes $X_{\Pa{i}} \subset \{X_1, \ldots, X_d\}$, based on some \emph{causal mechanism} $f_i$ involving besides $X_{\Pa{i}}$ some random variable $E_i$ drawn after distribution $\mathcal{E}$, meant to account for all unobserved variables. 

Letting $\cal G$ denote the causal graph obtained by drawing arrows from causes $X_{\Pa{i}}$ towards their effects $X_i$, we restrict ourselves to directed acyclic graphs (DAG), where the propagation of interventions to end nodes is assumed to be instantaneous.  This assumption suitably represents causal phenomena in cross-sectional studies. An example of functional causal model with five variables is illustrated on Fig. \ref{figure:causalnetwork}.

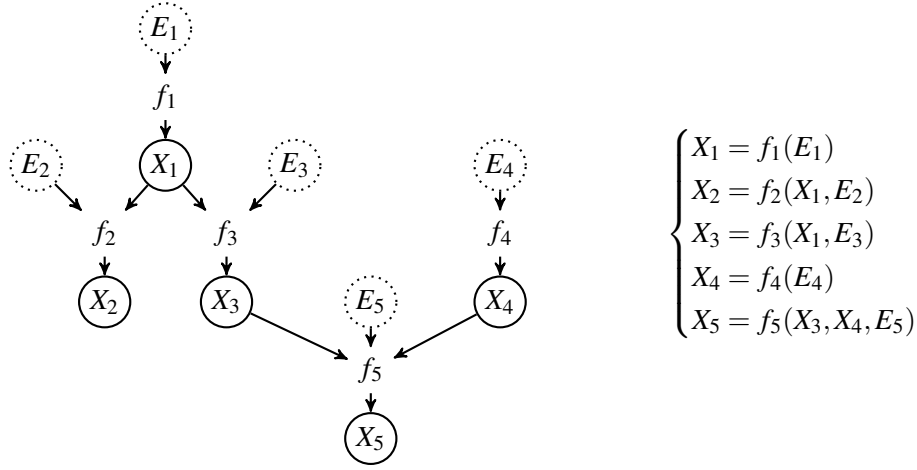
\begin{figure}[h]
    \centering
    \begin{tikzpicture}[node distance=0.25cm, auto,]
        \node[punkt, dotted] (e2) at (0,0) {$E_1$};
        \node[below=of e2] (f2) {$f_1$};
        \node[punkt, below=of f2] (x2) {$X_1$};
        \node[punkt, dotted, right=1cm of x2] (e4) {$E_3$};
        \node[punkt, dotted, left=1cm of x2] (e3) {$E_2$};
        \node[punkt, dotted, right=2cm of e4] (e5) {$E_4$};
        \node[below=of e5] (f5) {$f_4$};
        \node[punkt, below=of f5] (x5) {$X_4$};
        \node[punkt, dotted, left=1cm of x5] (e6) {$E_5$};
        \node[below=of x2, shift={(-0.8cm,0)}] (f3) {$f_2$};
        \node[right=1cm of f3] (f4) {$f_3$};
        \node[punkt, below=of f4] (x4) {$X_3$};
        \node[below=of e6] (f6) {$f_5$};
        \node[punkt, below=of f6] (x6) {$X_5$};
        \node[punkt, below=of f3] (x3) {$X_2$};

        \draw[pil] (e2) -- (f2);
        \draw[pil] (e3) -- (f3);
        \draw[pil] (e4) -- (f4);
        \draw[pil] (e5) -- (f5);
        \draw[pil] (e6) -- (f6);

        \draw[pil] (f2) -- (x2);
        \draw[pil] (f3) -- (x3);
        \draw[pil] (f4) -- (x4);
        \draw[pil] (f5) -- (x5);
        \draw[pil] (f6) -- (x6);

        \draw[pil] (x2) -- (f3);
        \draw[pil] (x2) -- (f4);
        \draw[pil] (x4) -- (f6);
        \draw[pil] (x5) -- (f6);

        \node[right=1.75cm of f5] {
                $\begin{cases}
                X_1 = f_1(E_1) \\
                X_2 = f_2(X_1,E_2) \\
                X_3 = f_3(X_1,E_3) \\
                X_4 = f_4(E_4) \\
                X_5 = f_5(X_3,X_4,E_5) \\
                \end{cases}$};
    \end{tikzpicture}
    \caption{Example of a Functional Causal Model (FCM) on $\mathbf{X} = [X_1, \ldots, X_5]$: Left: causal graph $\mathcal{G}$; right: causal mechanisms.}
    \label{figure:causalnetwork}
\end{figure}

\subsection{Notations} 

By abuse of notation and for simplicity, a variable $X$ and the associated node in the causal graph, in one-to-one correspondence, are noted in the same way. Variables $X$ and $Y$ are adjacent iff there exists an edge between both nodes in the graph. This edge can model i) a direct causal relationship ($ X \rightarrow Y$ or $Y  \rightarrow X)$; ii) a causal relationship in either direction ($X - Y$); iii) a
non-causal association ($X \leftrightarrow Y$) due to external common causes \citep{richardson2002ancestral}.\\
\textbf{\em Conditional independence}: $(X \indep Y|Z)$ is meant as variables $X$ and $Y$ are independent conditionally to $Z$, i.e. $P(X,Y|Z) = P(X|Z)P(Y|Z)$.\\
\textbf{\em V-structure, a.k.a. unshielded collider}: Three variables $\{X, Y, Z\}$ form a v-structure iff their causal structure is: $X \rightarrow Z \leftarrow Y$.\\
\textbf{\em Skeleton of the DAG}: the skeleton of the DAG  is the undirected graph obtained by replacing all edges by undirected edges.\\
\textbf{\em Markov equivalent DAG}: two DAGs with same skeleton and same v-structures  are said to be \textit{Markov equivalent} \citep{pearl1991formal}. A \textit{Markov equivalence class} is represented by a \textit{Completed Partially Directed Acyclic Graph} (CPDAG) having both directed and undirected edges.

\subsection{Assumptions and Properties} 

The state of the art in causal modeling most commonly involves four assumptions:\\
\textbf{\em Causal sufficiency assumption (CSA)}:  \X\ is said to be \textit{causally sufficient} if no pair of variables $\{X_i, X_j\}$ in \X\ has a common cause external to $\textbf{X}_{\backslash i,j}$.\\
\textbf{\em Causal Markov assumption (CMA)}: all variables are independent of their non-effects (non descendants in the causal graph) conditionally to their direct causes (parents) \citep{spirtes2000causation}. For an FCM, this assumption holds if the graph is a DAG and error terms $E_i$ in the FCM are independent on each other \citep{pearl2009causality}.\\
\textbf{\em Conditional independence relations in an FCM}:
if CMA applies, the data generated by the FCM satisfy all conditional independence (CI) relations among variables in \X\ via the notion of d-separation \citep{pearl2009causality}. CIs are called Markov properties. Note that there may be more CIs in data than present in the graph (see the Faithfulness assumption below). The joint distribution of the variables is expressed as the product of the distribution of each variable conditionally on its parents in the graph.\\
\textbf{\em Causal Faithfulness Assumption (CFA)}: the joint distribution $P(\X)$ is \textit{faithful} to the graph $\mathcal{G}$ of an FCM iff every conditional independence relation that holds true in $P$ is entailed by $\mathcal{G}$ \citep{spirtes2016causal}. Therefore, if  there  exists an independence  relation  in \X\  that  is not  a consequence  of  the Causal  Markov  assumption,  then  \X\ is  \textit{unfaithful} \citep{scheines1997introduction}. It follows from CMA and CFA that every causal path in the graph corresponds to a dependency between variables, and vice versa.\\
\textbf{\em V-structure property}. Under CSA, CMA and CFA, if variables $\{X, Y, Z\}$ satisfy: i)  $\{X, Y\}$ and $\{Y, Z\}$ are adjacent; ii) $\{X, Z\}$ are NOT adjacent; iii)  $X \nindep Z | Y$, then their causal structure  is a v-structure ($X \rightarrow Y \leftarrow Z$).

\section{State of the art \label{sec:soa}}

This section reviews methods to infer causal relationships, based on either the Markov properties of a DAG such as v-structures %or colliders, 
or asymmetries in the joint distributions of pairs of variables.

\subsection{Learning the CPDAG}

Structure learning methods classically use conditional independence (CI) relations in order  to identify  the  Markov equivalence class of the sought {Directed Acyclic Graph}, referred to as CPDAG, under CSA, CMA and CFA.

Considering the functional model on $\mathbf{X} = [X_1, \ldots, X_5]$ on Fig. \ref{figure:causalnetwork}, the associated DAG  $\mathcal{G}$ and graph skeleton are respectively depicted on Fig. \ref{MarkovEqDAG}(a) and (b). Causal modeling exploits observational data to recover the $\mathcal{G}$ structure from all CI (Markov properties) between variables.\footnote{The so-called constraint-based methods base the recovery of graph structure on conditional independence tests. In general, proofs of model identifiability assume the existence of an ``oracle" providing perfect knowledge of the CIs, i.e. {\em de facto} assuming an infinite amount of training data.} Under CSA, CMA and CFA, as $(X_3 \indep X_4|X_5)$ does not hold, a v-structure $X_3 \rightarrow X_5 \leftarrow X_4$ is identified (Fig. \ref{MarkovEqDAG}(c)). However, one also has $(X_1 \indep X_5|X_3)$ and $(X_2 \indep X_3|X_1)$. Thus the DAGs on Figs. \ref{MarkovEqDAG}(d) and (e) encode the same conditional independences as the true DAG (Fig. \ref{MarkovEqDAG}(a)). Therefore the true DAG cannot be fully identified based only on independence tests, and the edges between the pairs of nodes $\{X_1, X_2\}$ and $\{X_1, X_3\}$ are left undirected. The identification process thus yields the partially undirected graph depicted on Fig. \ref{MarkovEqDAG}(c), called \textit{Completed Partially Directed Acyclic Graph} (CPDAG).

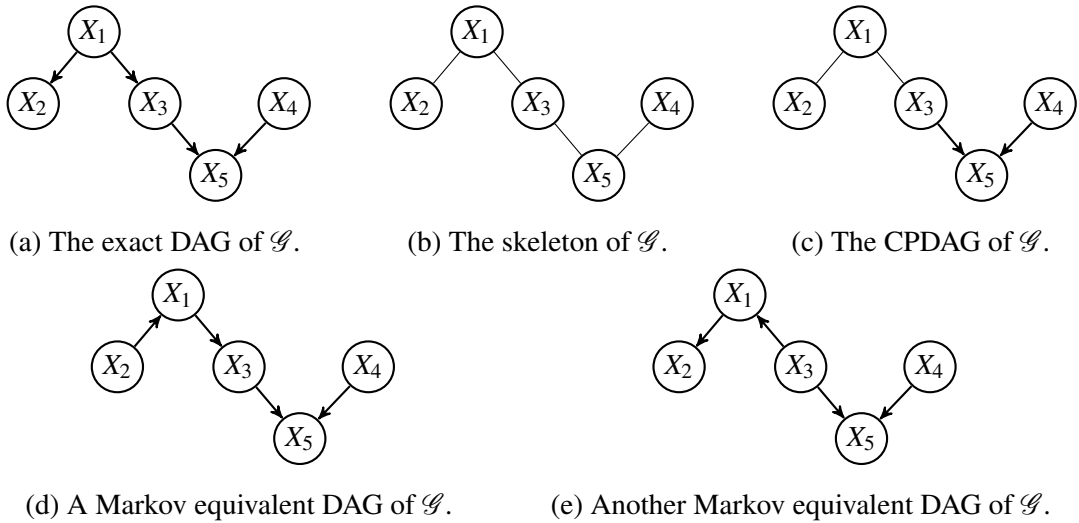
\begin{figure}[!h]
    \centering
    \begin{tikzpicture}[node distance=0.25cm, auto,]
        \node[punkt] (x1) at (0,0)  {$X_1$};
        \node[punkt, below=of x1, shift={(-0.8cm,0)}] (x2) {$X_2$};
        \node[punkt, below=of x1, shift={(0.8cm,0)}] (x3) {$X_3$};
        \node[punkt, below=of x3, shift={(0.8cm,0)}] (x5) {$X_5$};
        \node[punkt, right=1cm of x3] (x4) {$X_4$};
        \node[below=of x5, shift={(-0.8cm,0)}] (cap) {(a) The exact DAG of $\mathcal{G}$.};

        \draw[pil] (x1) -- (x2);
        \draw[pil] (x1) -- (x3);
        \draw[pil] (x3) -- (x5);
        \draw[pil] (x4) -- (x5);
    \end{tikzpicture}
    \hspace{2pc}
    \begin{tikzpicture}[node distance=0.25cm, auto,]
        \node[punkt] (x1) at (0,0)  {$X_1$};
        \node[punkt, below=of x1, shift={(-0.8cm,0)}] (x2) {$X_2$};
        \node[punkt, below=of x1, shift={(0.8cm,0)}] (x3) {$X_3$};
        \node[punkt, below=of x3, shift={(0.8cm,0)}] (x5) {$X_5$};
        \node[punkt, right=1cm of x3] (x4) {$X_4$};
        \node[below=of x5, shift={(-0.8cm,0)}] (cap) {(b) The skeleton of $\mathcal{G}$.};

        \draw (x3) -- (x5);
        \draw (x4) -- (x5);
        \draw (x1) -- (x2);
        \draw (x1) -- (x3);

    \end{tikzpicture}
    \hspace{2pc}
    \begin{tikzpicture}[node distance=0.25cm, auto,]
        \node[punkt] (x1) at (0,0)  {$X_1$};
        \node[punkt, below=of x1, shift={(-0.8cm,0)}] (x2) {$X_2$};
        \node[punkt, below=of x1, shift={(0.8cm,0)}] (x3) {$X_3$};
        \node[punkt, below=of x3, shift={(0.8cm,0)}] (x5) {$X_5$};
        \node[punkt, right=1cm of x3] (x4) {$X_4$};
        \node[below=of x5, shift={(-0.8cm,0)}] (cap) {(c) The CPDAG of $\mathcal{G}$.};

        \draw[pil] (x3) -- (x5);
        \draw[pil] (x4) -- (x5);
        \draw (x1) -- (x2);
        \draw (x1) -- (x3);
    \end{tikzpicture}
     \hspace{2pc}
    \begin{tikzpicture}[node distance=0.25cm, auto,]
        \node[punkt] (x1) at (0,0)  {$X_1$};
        \node[punkt, below=of x1, shift={(-0.8cm,0)}] (x2) {$X_2$};
        \node[punkt, below=of x1, shift={(0.8cm,0)}] (x3) {$X_3$};
        \node[punkt, below=of x3, shift={(0.8cm,0)}] (x5) {$X_5$};
        \node[punkt, right=1cm of x3] (x4) {$X_4$};
        \node[below=of x5, shift={(-0.8cm,0)}] (cap) {(d) A Markov equivalent DAG of $\mathcal{G}$.};

        \draw[pil] (x2) -- (x1);
        \draw[pil] (x1) -- (x3);
        \draw[pil] (x3) -- (x5);
        \draw[pil] (x4) -- (x5);
    \end{tikzpicture}
    \hspace{2pc}
    \begin{tikzpicture}[node distance=0.25cm, auto,]
        \node[punkt] (x1) at (0,0)  {$X_1$};
        \node[punkt, below=of x1, shift={(-0.8cm,0)}] (x2) {$X_2$};
        \node[punkt, below=of x1, shift={(0.8cm,0)}] (x3) {$X_3$};
        \node[punkt, below=of x3, shift={(0.8cm,0)}] (x5) {$X_5$};
        \node[punkt, right=1cm of x3] (x4) {$X_4$};
        \node[below=of x5, shift={(-0.8cm,0)}] (cap) {(e) Another Markov equivalent DAG of $\mathcal{G}$.};

        \draw[pil] (x1) -- (x2);
        \draw[pil] (x3) -- (x1);
        \draw[pil] (x3) -- (x5);
        \draw[pil] (x4) -- (x5);
    \end{tikzpicture}
    \caption{Example of a Markov equivalent class. There exists three graphs (a, d, e) consistent with a given graph skeleton (b); the set of these consistent graphs defines the Markov equivalent class (c).} 
    \label{MarkovEqDAG}
\end{figure}

The  main three families of methods used to recover the CPDAG of an FCM with continuous data are constraint-based methods, score-based methods, and hybrid methods \citep{drton2016structure}. 

\subsubsection{Constraint-based methods} 

Constraint-based methods exploit conditional independences between variables to identify all v-structures.  One of the most well-known constraint-based algorithms is the PC algorithm \citep{spirtes1993search}. PC first builds the DAG skeleton based on conditional independences among variables and subsets of variables. Secondly, it identifies v-structures (Fig. \ref{MarkovEqDAG}(c)). Finally, it uses propagation rules to orient remaining edges, avoiding the creation of directed cycles or new v-structures. Under CSA, CMA and CFA, and assuming an oracle indicating all conditional independences, PC returns the CPDAG of the functional causal model. In practice, PC uses statistical tests to accept or reject conditional independence at a given confidence level. Besides mainstream tests (e.g., s Z-test or T-Test for continuous Gaussian variables, and $\chi$-squared or G-test for categorical variables), non-parametric independence tests based on machine learning are becoming increasingly popular, such as kernel-based conditional independence tests \citep{zhang2012kernel}.
The FCI algorithm \citep{spirtes1999algorithm} extends PC; it relaxes the  \textit{causal sufficiency} assumption and deals with latent variables. The RFCI algorithm \citep{colombo2012learning} is faster than FCI and handles high-dimensional DAGs with latent variables.
Achilles' heel of constraint-based algorithms is their reliance on conditional independence tests. The CI accuracy depends on the amount of available data, with exponentially increasing size with the number of variables. Additionally, the use of propagation rules to direct edges is prone to error propagation.

\subsubsection{Score-based methods} 

Score-based methods explore the space of CPDAGs and minimize a global score. For example, the space of graph structures is explored using operators (\textit{add edge}, \textit{remove  edge}, and \textit{reverse edge}) by the Greedy Equivalent Search (GES) algorithm \citep{chickering2002optimal}, returning the optimal structure in the sense of  the Bayesian Information Criterion.\footnote{After \cite{ramsey2015scaling}, in the linear model with Gaussian variable case the individual BIC score to minimize for a variable $X$ given its parents is up to a constant $n~\text{ln}(s)+ c~k~\text{ln}(n)$, where $n~\text{ln}(s)$ is the likelihood term, with $s$ the residual variance after regressing $X$ onto its parents, and $n$ the number of data samples. $c~k~\text{ln}(n)$ is a penalty term for the complexity of the graph (here  the number of edges). $k = 2p + 1$, with $p$ the total number of parents of the variable $X$ in the graph. $c=2$ by default, chosen empirically. The global score minimized by the algorithm is the sum over all variables of the individual BIC score given the parent variables in the graph.}

In order to find the optimal CPDAG corresponding to the minimum score, the GES algorithm starts with an empty graph. A first forward phase is performed, iteratively adding edges to the model in order to improve the global score.  A second backward phase iteratively removes edges to improve the score. Under CSA, CMA and CFA, GES identifies the true CPDAG in the large sample limit, if the score used is decomposable, score-equivalent and consistent \citep{chickering2002optimal}. More recently, \cite{ramsey2015scaling} proposed a GES extension called Fast Greedy Equivalence Search (FGES) algorithm.  FGES uses the same scores and search algorithm with different data structures; it greatly speeds up GES by caching information about scores during each phase of the process.

\subsubsection{Hybrid algorithms} 

Hybrid algorithms combine ideas from constraint-based and score-based algorithms. According to \cite{nandy2015high}, such methods often use a greedy search like the GES method on a restricted search space for the sake of computational efficiency. This  restricted space is defined using conditional independence tests.   For instance the Max-Min Hill climbing (MMHC) algorithm \citep{tsamardinos2006max} firstly builds the skeleton of a Bayesian network using conditional independence tests and then performs a Bayesian-scoring greedy hill-climbing search to orient the edges. 
The Greedy Fast Causal Inference (GFCI) algorithm proceeds in the other way around, using FGES to get rapidly a first sketch of the graph (shown to be more accurate than those obtained with constraint-based methods), then using the FCI constraint-based rules to orient the edges in presence of potential confounders \citep{ogarrio2016hybrid}.

\subsection{Exploiting asymmetry between cause and effect}

The abovementioned score-based and constraint-based methods do not take into account the full information from the observational data \citep{spirtes2016causal}, such as data asymmetries induced by the causal directions.

\subsubsection{The intuition} 

Let us consider FCM $Y = X + E$, with $E$ a random noise { independent} of $X$ by construction.  Graph constraints cannot orient the $X - Y$ edge as both graphs $X \rightarrow Y$ and $Y \rightarrow X$ are Markov equivalent. However, the implicit v-structure $X\rightarrow Y \leftarrow E$ can be exploited provided that either $X$ or $E$ does not follow a {\bf Gaussian distribution}. Consider the linear regression $Y = a X + b$ (blue curve in Fig. \ref{figure:parallelogram_example}); the residual is independent of $X$. Quite the contrary, the residual of the linear regression $X = a' Y + b'$ (red curve in Fig. \ref{figure:parallelogram_example}) is {\em not} independent of $Y$ as far as the independence of the error term holds true \citep{shimizu2006linear}. In this toy example, the asymmetries in the joint distribution of $X$ and $Y$ can be exploited to recover the causal direction $X \rightarrow Y$ \citep{spirtes2016causal}. 

 \begin{figure}[h!]
\begin{center}
\includegraphics[width=4cm]{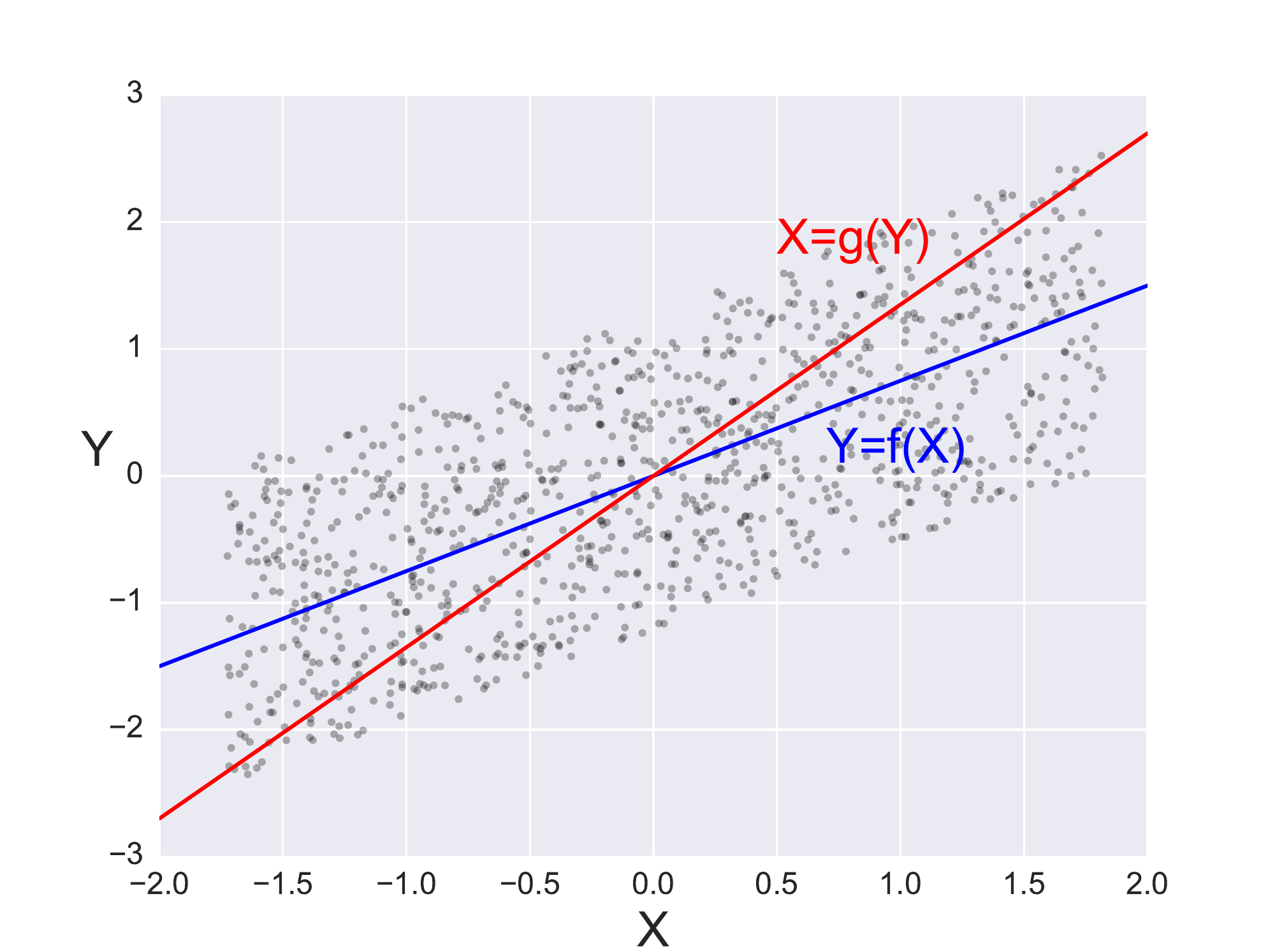}
\includegraphics[width=4cm]{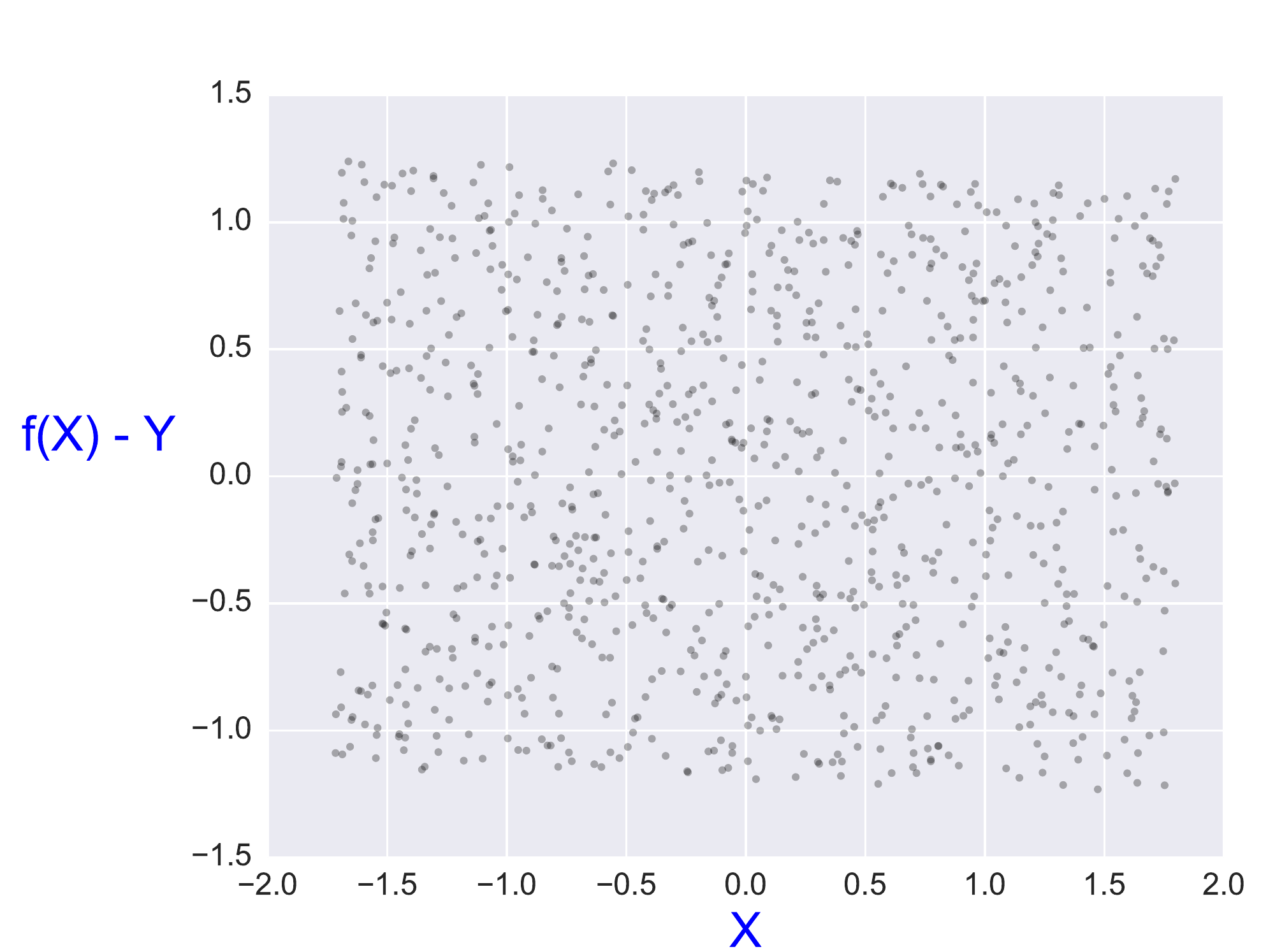}
\includegraphics[width=4cm]{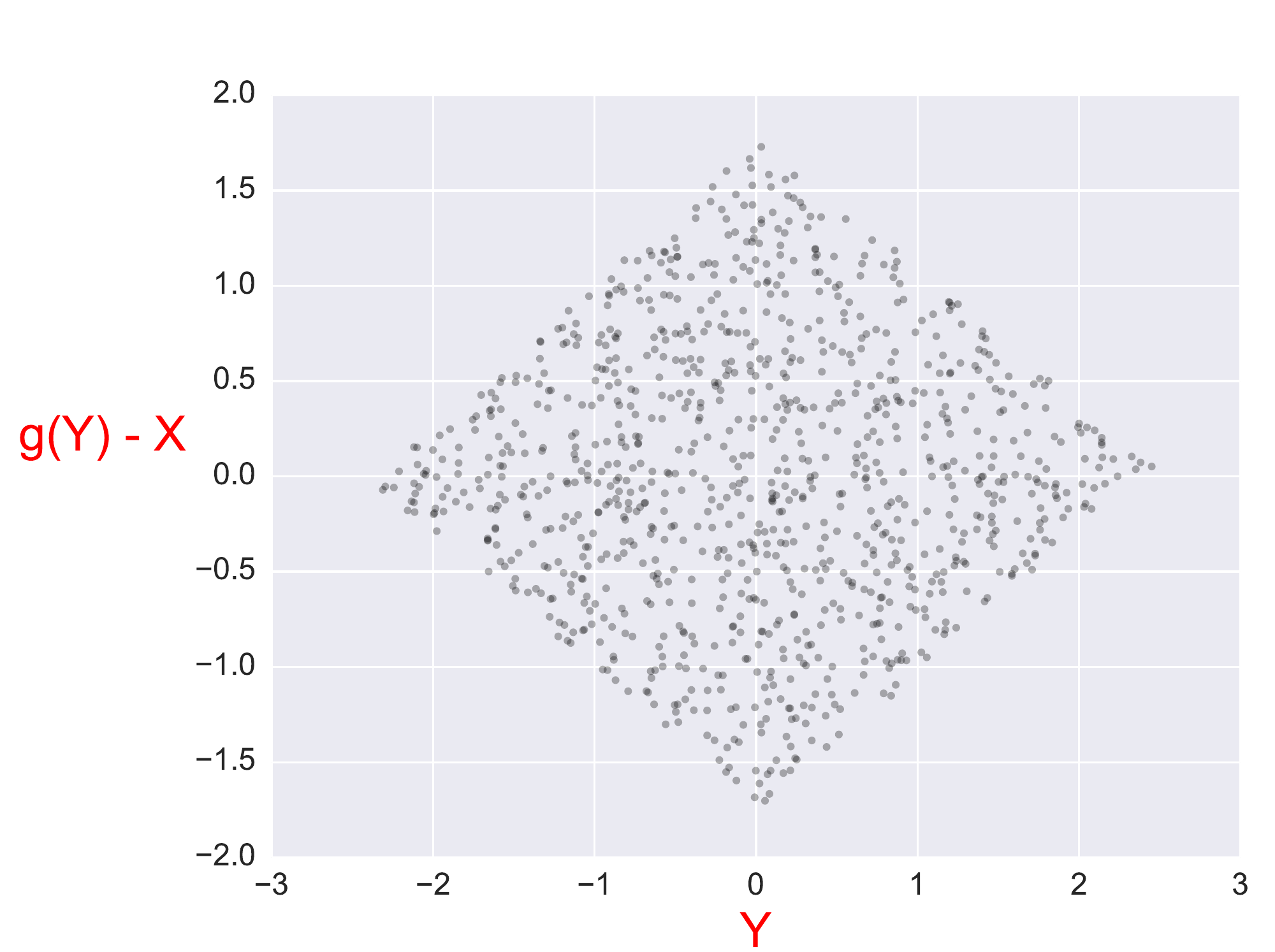}
\caption{Left: Joint distribution $P(X,Y)$ generated from DAG $X \rightarrow Y + E$, with E a uniform noise variable. The linear regression of $Y$ on $X$ (respectively of $X$ on $Y$) is depicted as a blue (resp. red) curve. Middle:  Error $f(X) - Y$ is independent of $X$. Right: Error $g(Y) - X$ is not independent of $Y$. The asymmetry establishes that the true causal model is $X \rightarrow Y$. Better seen in color.}
\label{figure:parallelogram_example}
\end{center}
\end{figure}

\subsubsection{Restriction on the class of causal mechanisms considered}

Causal inference is bound to rely on assumptions such as non-Gaussianity or additive noise. In the absence of any such assumption, \cite{zhang2016estimation}  show that, even in the bivariate case, for any function $f$ and noise variable $E$ independent of $X$ such that $Y = f(X,E)$, it is always feasible to construct some $\tilde{f}$ and  $\tilde{E}$, with $\tilde{E}$ independent of $Y$, such that $X = \tilde{f}(Y,\tilde{E})$. 
An alternative, supporting asymmetry detection and hinting at a causal direction, is based on restricting the class of functions $f$ (e.g. only considering regular functions). According to \cite{quinn2011learning}, the first approach in this direction is LiNGAM  \citep{shimizu2006linear}. LiNGAM handles linear structural equation models, where each variable is continuous and modeled as:

\begin{equation}
X_i = \sum_k \alpha_k P_a^{k}(X_i) + E_i,  i \in \llbracket 1,n \rrbracket
\label{SEM}
\end{equation}

with  $P_a^{k}(X_i)$ the $k$-th parent of $X_i$ and $\alpha_k$ a real value.
Assuming further that all probability distributions of source nodes in the causal graph are non-Gaussian, \cite{shimizu2006linear} show that the causal structure is fully identifiable (all edges can be oriented).

\subsubsection{Pairwise methods} 

In the continuous, non-linear bivariate case, specific methods have been developed to orient the variable edge.\footnote{These methods can be extended to the multivariate case and used for causal graph identification by orienting each edge in turn.}
A well known example of bivariate model is the additive noise model (ANM) \citep{hoyer2009nonlinear}, with data generative model  $Y = f(X) + E$, $f$ a (possibly non-linear) function and $E$ a noise independent of $X$. The authors prove the identifiability of the ANM in the following sense: if $P(X,Y)$ is consistent with ANM $Y = f(X) + E$, then i) there exists no AMN $X = g(Y) + E'$ consistent with $P(X,Y)$; ii) the true causal direction is $X \rightarrow Y$. Under the independence assumption between $E$ and $X$, the ANM admits a single non-identifiable case, the linear model with Gaussian input and Gaussian noise \citep{mooij2016distinguishing}.

A more general model is the post-nonlinear model (PNL) \citep{zhang2009identifiability}, involving an additional nonlinear function on the top of an additive noise: $Y = g(f(X)+E)$, with $g$ an invertible function. The price to pay for this higher generality is an increase in the number of non identifiable cases.

The Gaussian Process Inference model (GPI)  \citep{stegle2010probabilistic} infers the causal direction without explicitly restricting the class of possible causal mechanisms. The authors build two Bayesian generative models, one for $X \rightarrow Y$ and one for $Y \rightarrow X$, where the distribution of the cause is modeled with a Gaussian mixture model, and the causal mechanism $f$ is a Gaussian process.  The causal direction is determined from the generative model best fitting the data (maximizing the data likelihood). Identifiability here follows from restricting the underlying class of functions and enforcing their smoothness (regularity).  Other causal inference methods \citep{sgouritsa2015inference}  are based on the idea that if $X \rightarrow Y$, the marginal probability distribution of the cause $P(X)$ is independent of the causal mechanism $P(Y|X)$, hence estimating  $P(Y|X)$ from $P(X)$ should hardly be possible, while estimating $P(X|Y)$ based on $P(Y)$ may be possible.  The reader is referred to \cite{statnikov2012new} and \cite{mooij2016distinguishing} for a thorough review and benchmark of the pairwise methods in the bivariate case.

A new ML-based approach tackles causal inference as a pattern recognition problem. This setting was introduced in the Causality challenges \citep{guyon2013cepc,IGuyon2014}, which released 16,200 pairs of variables $\{X_i,Y_i\}$, each pair being described by a sample of their joint distribution, and labeled with the true $\ell_i$ value of their causal relationship, with $\ell_i$ ranging in $\{$ $X_i \rightarrow Y_i$, $Y_i \rightarrow X_i$,  $X_i \indep  Y_i$, $X_i \leftrightarrow Y_i$ (presence of a confounder) $\}$. The causality classifiers trained from the challenge pairs yield encouraging results on test pairs. The limitation of this ML-based causal modeling approach is that causality classifiers intrinsically depend on the representativity of the training pairs, assumed to be drawn from a same ``Mother distribution'' \citep{lopez2015towards}.

Note that bivariate methods can be used to uncover the full DAG, and independently orient each edge, with the advantage that an error on one edge does not propagate to the rest of the graph (as opposed to constraint and score-based methods). However, bivariate methods do not leverage the full information available in the dependence relations. For example in the linear Gaussian case (linear model and Gaussian distributed inputs and noises), if a triplet of variables $\{A, B, C\}$ is such that $A, B$ (respectively $B, C$) are dependent on each other but $A \indep C)$, a constraint-based method would identify the v-structure $A \rightarrow B \leftarrow C$ (unshielded collider); 
still, a bivariate model based on cause-effect asymmetry  would neither identify $A \rightarrow B$ nor $B \leftarrow C$.\\

\subsection{Discussion}
This brief survey has shown the complementarity of CPDAG and pairwise methods. The former ones can at best return partially directed graphs; the latter ones do not optimally exploit the interactions between all variables. 

To overcome these limitations, an extension of the bivariate post-nonlinear model (PNL) has been proposed \citep{zhang2009identifiability}, where an FCM is trained for any plausible causal structure, and each model is tested {\em a posteriori} for the required independence between errors and  causes. The main PNL limitation is its super-exponential cost with  the  number  of  variables \citep{zhang2009identifiability}. Another hybrid approach uses a constraint based algorithm to identify a Markov equivalence class, and thereafter uses bivariate modelling to orient the remaining  edges \citep{zhang2009identifiability}. For example, the constraint-based PC algorithm can identify the v-structure $X_3 \rightarrow X_5 \leftarrow X_4$ in an FCM (Fig. \ref{MarkovEqDAG}), enabling the bivariate PNL method to further infer the remaining arrows $X_1  \rightarrow X_2$ and $X_1 \rightarrow X_3$. Note that an effective combination of constraint-based and bivariate approaches requires a final verification phase to test the consistency between the  v-structures and the edge orientations.

This paper aims to propose a unified framework getting the best out of both worlds of CPDAG and bivariate approaches.

An inspiration of the approach is the CAM algorithm \citep{buhlmann2014cam}, which is an extension to the graph setting of the pairwise additive model (ANM) \citep{hoyer2009nonlinear}. In CAM the FCM is modeled as:
\begin{equation}
X_i = \sum_{k \in \Pa{i}} f_k(X_k) + E_i, \mbox{ for } i=1,\ldots, d
\label{eq:cam}
\end{equation}

Our method can be seen an extension of CAM, as it allows non-additive noise terms and non-additive contributions of causes, in order to model flexible conditional distributions, and addresses the problem of learning FCMs (Section \ref{sec:setting}):

\begin{equation}   {X}_i  = {f}_i({X}_{\Pa{i}}, {E}_i), \mbox{ for } i=1,\ldots, d
\end{equation}

An other inspiration of our framework is the recent method of \citet{lopez2016revisiting}, where a conditional generative adversarial network is trained to model $X \rightarrow Y$ and $Y \rightarrow X$ in order to infer the causal direction based on the Occam's razor principle.

This approach, called {\bf Causal Generative Neural Network (CGNN)}, features two original contributions. Firstly, multivariate causal mechanisms $f_i$ are learned as {\bf generative neural networks} (as opposed to, regression networks). The novelty is to use neural nets to model the joint distribution of the observed variables and learn a continuous FCM. This approach does not explicitly restrict the class of functions used to represent the causal models (see also \citep{stegle2010probabilistic}), since neural networks are universal approximators.  Instead, a regularity argument is used to enforce identifiability, in the spirit of supervised learning: the methods searches a trade-off between data fitting and model complexity. 

Secondly, the data generative models are trained using a non-parametric score, the Maximum Mean Discrepancy \citep{gretton2007kernel}. This criterion is used instead of likelihood based criteria, hardly suited to complex data structures, or mean square criteria, implicitly assuming an additive noise (e.g. as in CAM, Eq. \ref{eq:cam}).

Starting from a known skeleton, Section \ref{sec:CG} presents a version of the proposed approach under the usual Markov, faithfulness, and causal sufficiency assumptions. The empirical validation of the approach is detailed in Section \ref{sec:expe}.  In Section \ref{confounder}, the causal sufficiency assumption is relaxed and the model is extended to handle possible hidden confounding factors. Section \ref{discussion} concludes the paper with some perspectives for future work.

  \section{Causal Generative Neural Networks \label{sec:CG}}

Let $\textbf{X} = [X_1, \ldots, X_d]$ denote a set of continuous random variables with joint distribution $P$, and further assume that the joint density function $h$ of $P$ is continuous and strictly positive on a compact subset of $\mathbb{R}^{d}$ and zero elsewhere.

This section first presents the modeling of continuous FCMs with generative neural networks with a given graph structure (Section \ref{model_CGNN}), the evaluation of a candidate  model (Section \ref{ModelEval}), and finally, the learning of a best candidate from observational data (Section \ref{sec:optim}).

\subsection{Modeling continuous FCMs with generative neural networks \label{model_CGNN}}

We first show that there exists a (non necessarily unique) \textit{continuous} functional causal model $(\mathcal{G}, f, \mathcal{E})$ such that the associated data generative process fits the distribution $P$ of the observational data. 

\begin{prop}{\label{prop1}}
 Let $\textbf{X} = [X_1, \ldots, X_d]$ denote a set of continuous random variables with joint distribution $P$, and further assume that the joint density function $h$ of $P$ is continuous and strictly positive on a compact and convex subset of $\mathbb{R}^{d}$, and zero elsewhere. Letting $\cal G$ be a DAG such that $P$ can be factorized along $\cal G$, 
 $$ P(X) = \prod_i P(X_i | X_{\Pa{i}})$$
 there exists $f = (f_1, \ldots, f_d)$ with $f_i$ a continuous function with compact support in $\mathbb{R}^{|\Pa{i}|}\times [0,1]$ such that $P(X)$ equals the generative model defined from FCM $({\cal G}, f, {\cal E})$, with ${\cal E} = \mathcal{U}[0,1]$ the uniform distribution on $[0,1]$.
\end{prop}
\begin{proof}
In Appendix \ref{app:proofs}
\end{proof}

In order to model such continuous FCM $({\cal G}, f, {\cal E})$ on $d$ random variables $\textbf{X} = [X_1, \ldots, X_d]$, we introduce the CGNN (Causal Generative Neural Network) depicted on Figure \ref{figure:CGNN_FCM}.

\begin{definition}
A CGNN over d variables $[\hat{X}_1, \ldots, \hat{X}_d]$ is a triplet $\mathcal{C}_{\widehat{\mathcal{G}},\hat{f}} = (\widehat{\mathcal{G}}, \hat{f}, {\cal E})$ where:
\begin{enumerate}
\item $\widehat{\mathcal{G}}$ is a Directed Acyclic Graph (DAG) associating to each variable $\hat{X}_i$ its set of parents noted $\hat{X}_{\Pahat{i}}$ for $i \in [[1,d]]$
\item For $i \in \llbracket 1,d \rrbracket$, causal mechanism $\hat{f}_i$ is a 1-hidden layer regression neural network with $n_h$ hidden neurons: 

\begin{equation}
\hat{X}_i  =  \hat{f}_i(\hat{X}_{\Pahat{i}}, E_i) 
 =  \sum_{k=1}^{n_h} \bar{w}^{i}_k \sigma \left(\sum_{j \in \Pa{i}} \hat{w}^{i}_{jk} \hat{X}_j 
 + w^{i}_{k} E_i +  b^{i}_k \right) +\bar{b}^{i} 
 \label{eq:fcm_neural}
\end{equation}
with $n_h \in \mathbb{N}*$  the number of hidden units, $\bar{w}^{i}_k,\hat{w}^{i}_{jk},w^{i}_{k},b^{i}_k,\bar{b}^{i} \in \mathbb{R}$ the parameters of the neural network,    and $\sigma$ a continuous activation function . 
\item Each variable $E_i$ is independent of the \textit{cause} $X_i$. Furthermore, all noise variables are mutually independent and drawn after same distribution $\mathcal{E}$. 
\end{enumerate}
\end{definition}

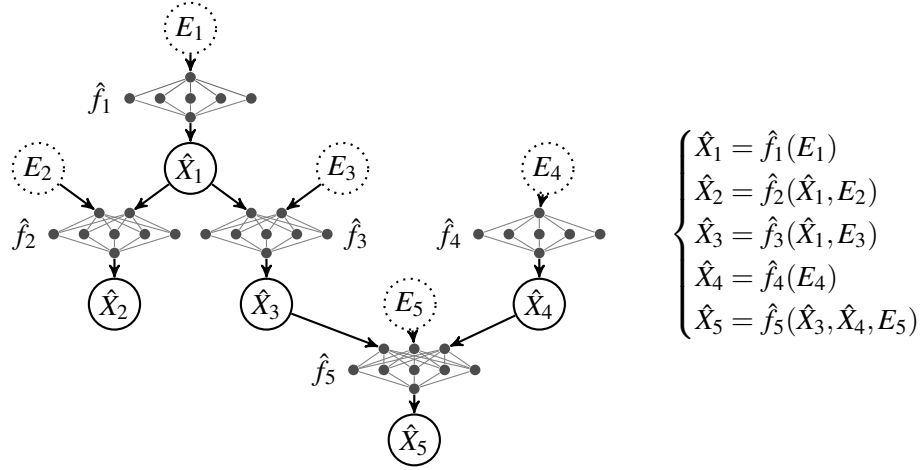
\begin{figure}[!h]
    \centering
    
\begin{tikzpicture}[node distance=0.25cm, auto,]
         \tikzstyle{neuron}=[circle,fill=black!70,minimum size=4pt,inner sep=0pt]
        \tikzstyle{input neuron}=[neuron];
        \tikzstyle{output neuron}=[neuron];
        \tikzstyle{hidden neuron}=[neuron];
        \tikzstyle{annot} = [text width=4em, text centered]

                \begin{scope}[shift={(1.745,-4.65)}]
            % Draw the input layer nodes
    \foreach \name / \y in {1,...,3}
    % This is the same as writing \foreach \name / \y in {1/1,2/2,3/3,4/4}
        \node[xshift=.4cm, yshift=.4cm, input neuron] (I5-\name) at (\y * \inlsep, 0) {};

    % Draw the hidden layer nodes
    \foreach \name / \y in {1,...,5}
        \path[yshift=0.5cm]
            node[hidden neuron] (H5-\name) at (\y * \inlsep cm,- \layersep) {};

    % Draw the output layer node
    \node[output neuron, below of=H5-3] (O5) {};

    % Connect every node in the input layer with every node in the
    % hidden layer.
    \foreach \source in {1,...,3}
        \foreach \dest in {1,...,5}
            \path[draw=black!50] (I5-\source) edge (H5-\dest);

    % Connect every node in the hidden layer with the output layer
    \foreach \source in {1,...,5}
        \path[draw=black!50] (H5-\source) edge (O5);
                \end{scope}

                \begin{scope}[shift={(3.39,-2.85)}]
            % Draw the input layer nodes
    \foreach \name / \y in {1,...,1}
    % This is the same as writing \foreach \name / \y in {1/1,2/2,3/3,4/4}
        \node[xshift=.8cm, yshift=.4cm, input neuron] (I4-\name) at (\y * \inlsep, 0) {};

    % Draw the hidden layer nodes
    \foreach \name / \y in {1,...,5}
        \path[yshift=0.5cm]
            node[hidden neuron] (H4-\name) at (\y * \inlsep cm,- \layersep) {};

    % Draw the output layer node
    \node[output neuron, below of=H4-3] (O4) {};

    % Connect every node in the input layer with every node in the
    % hidden layer.
    \foreach \source in {1,...,1}
        \foreach \dest in {1,...,5}
            \path[draw=black!50] (I4-\source) edge (H4-\dest);

    % Connect every node in the hidden layer with the output layer
    \foreach \source in {1,...,5}
        \path[draw=black!50] (H4-\source) edge (O4);
                \end{scope}

                \begin{scope}[shift={(-1.2,-1.05)}]
            % Draw the input layer nodes
    \foreach \name / \y in {1,...,1}
    % This is the same as writing \foreach \name / \y in {1/1,2/2,3/3,4/4}
        \node[xshift=.8cm, yshift=.4cm, input neuron] (I1-\name) at (\y * \inlsep, 0) {};

    % Draw the hidden layer nodes
    \foreach \name / \y in {1,...,5}
        \path[yshift=0.5cm]
            node[hidden neuron] (H1-\name) at (\y * \inlsep cm,- \layersep) {};

    % Draw the output layer node
    \node[output neuron, below of=H1-3] (O1) {};

    % Connect every node in the input layer with every node in the
    % hidden layer.
    \foreach \source in {1,...,1}
        \foreach \dest in {1,...,5}
            \path[draw=black!60] (I1-\source) edge (H1-\dest);

    % Connect every node in the hidden layer with the output layer
    \foreach \source in {1,...,5}
        \path[draw=black!50] (H1-\source) edge (O1);
                \end{scope}
                
        \begin{scope}[shift={(-.2,-2.85)}]
            % Draw the input layer nodes
    \foreach \name / \y in {1,...,2}
    % This is the same as writing \foreach \name / \y in {1/1,2/2,3/3,4/4}
        \node[xshift=.6cm, yshift=.4cm, input neuron] (I3-\name) at (\y * \inlsep, 0) {};

    % Draw the hidden layer nodes
    \foreach \name / \y in {1,...,5}
        \path[yshift=0.5cm]
            node[hidden neuron] (H3-\name) at (\y * \inlsep cm,- \layersep) {};

    % Draw the output layer node
    \node[output neuron, below of=H3-3] (O3) {};

    % Connect every node in the input layer with every node in the
    % hidden layer.
    \foreach \source in {1,...,2}
        \foreach \dest in {1,...,5}
            \path[draw=black!50] (I3-\source) edge (H3-\dest);

    % Connect every node in the hidden layer with the output layer
    \foreach \source in {1,...,5}
        \path[draw=black!50] (H3-\source) edge (O3);
        \end{scope}

\begin{scope}[shift={(-2.2,-2.85)}]
            % Draw the input layer nodes
    \foreach \name / \y in {1,...,2}
    % This is the same as writing \foreach \name / \y in {1/1,2/2,3/3,4/4}
        \node[xshift=.6cm, yshift=.4cm, input neuron] (I-\name) at (\y * \inlsep, 0) {};

    % Draw the hidden layer nodes
    \foreach \name / \y in {1,...,5}
        \path[yshift=0.5cm]
            node[hidden neuron] (H-\name) at (\y * \inlsep cm,- \layersep) {};

    % Draw the output layer node
    \node[output neuron, below of=H-3] (O2) {};

    % Connect every node in the input layer with every node in the
    % hidden layer.
    \foreach \source in {1,...,2}
        \foreach \dest in {1,...,5}
            \path[draw=black!50] (I-\source) edge (H-\dest);

    % Connect every node in the hidden layer with the output layer
    \foreach \source in {1,...,5}
        \path[draw=black!50] (H-\source) edge (O2);
        \end{scope}
        \node[punkt, dotted] (e1) at (0,0) {$E_1$};
        %\node[draw, below=of e1] (f1) {  $\hat{f}_1$   };
        \node[punkt, below=of O1] (x1) {$\hat{X}_1$};
        \node[punkt, dotted, right=1.3cm of x1] (e3) {$E_3$};

        \node[punkt, dotted, left=1.3cm of x1] (e2) {$E_2$};
        \node[punkt, dotted, right=2cm of e3] (e4) {$E_4$};
        \node[punkt, below=of O4] (x4) {$\hat{X}_4$};
        \node[punkt, dotted, left=1cm of x4] (e5) {$E_5$};
        \node[punkt, below=of O3] (x3) {$\hat{X}_3$};
        \node[punkt, below=of O5] (x5) {$\hat{X}_5$};
        \node[punkt, below=of O2] (x2) {$\hat{X}_2$};
        
        \draw[pil] (e1) -- (I1-1);
        \draw[pil] (e2) -- (I-1);
        \draw[pil] (e3) -- (I3-2);
        \draw[pil] (e4) -- (I4-1);
        \draw[pil] (e5) -- (I5-2);

        \draw[pil] (O1) -- (x1);
        \draw[pil] (O2) -- (x2);
        \draw[pil] (O3) -- (x3);
        \draw[pil] (O4) -- (x4);
        \draw[pil] (O5) -- (x5);

        \draw[pil] (x1) -- (I-2);
        \draw[pil] (x1) -- (I3-1);
        \draw[pil] (x3) -- (I5-1);
        \draw[pil] (x4) -- (I5-3);

        \node[annot,left= -.6cm of H-1]{$\hat{f_2}$};
        \node[annot,left= -.6cm of H1-1]{$\hat{f_1}$};
        \node[annot,right= -.6cm of H3-5]{$\hat{f_3}$};
        \node[annot,left= -.6cm of H4-1]{$\hat{f_4}$};
        \node[annot,left= -.6cm of H5-1]{$\hat{f_5}$};

%     \begin{tikzpicture}[node distance=0.25cm, auto,]
%         \node[punkt, dotted] (e1) at (0,0) {$E_1$};
%         \node[draw, below=of e1] (f1) {  $\hat{f}_1$   };
%         \node[punkt, below=of f2] (x1) {$X_1$};
%         \node[punkt, dotted, right=1.3cm of x1] (e3) {$E_3$};
%         \node[punkt, dotted, left=1.3cm of x1] (e2) {$E_2$};
%         \node[punkt, dotted, right=2cm of e3] (e4) {$E_4$};
%         \node[draw, below=of e4] (f4) { $\hat{f}_4$};
%         \node[punkt, below=of f4] (x4) {$X_4$};
%         \node[punkt, dotted, left=1cm of x4] (e5) {$E_5$};
%         \node[draw, below=of x1, shift={(-1cm,0)}] (f2) {$\hat{f}_2$};
%         \node[draw,  below=of x1, shift={(1cm,0)}] (f3) {$\hat{f}_3$};
%         \node[punkt, below=of f3] (x3) {$X_3$};
%         \node[draw, below=of e5] (f5) {$\hat{f}_5$};
%         \node[punkt, below=of f5] (x5) {$X_5$};
%         \node[punkt, below=of f2] (x2) {$X_2$};
        
%         \draw[pil] (e1) -- (f1);
%         \draw[pil] (e2) -- (f2);
%         \draw[pil] (e3) -- (f3);
%         \draw[pil] (e4) -- (f4);
%         \draw[pil] (e5) -- (f5);

%         \draw[pil] (f1) -- (x1);
%         \draw[pil] (f2) -- (x2);
%         \draw[pil] (f3) -- (x3);
%         \draw[pil] (f4) -- (x4);
%         \draw[pil] (f5) -- (x5);

%         \draw[pil] (x1) -- (f2);
%         \draw[pil] (x1) -- (f3);
%         \draw[pil] (x3) -- (f5);
%         \draw[pil] (x4) -- (f5);

        \node[right=0.7cm of H4-5] {
                $\begin{cases}   
                \hat{X}_1 = \hat{f}_1(E_1) \\
                \hat{X}_2 = \hat{f}_2(\hat{X}_1,E_2) \\
                \hat{X}_3 = \hat{f}_3(\hat{X}_1,E_3) \\
                \hat{X}_4 = \hat{f}_4(E_4) \\
                \hat{X}_5 = \hat{f}_5(\hat{X}_3,\hat{X}_4,E_5) \\
                \end{cases}$};
    \end{tikzpicture}
    \caption{Left: Causal Generative Neural Network over variables $\hat{\mathbf{X}} = (\hat{X}_1, \ldots, \hat{X}_5)$. Right: Corresponding Functional Causal Model  equations.}
    \label{figure:CGNN_FCM}
\end{figure}
\def\z{\mbox{${\bf z}$}}
\def\zp{\mbox{$\hat{\bf z}$}}

It is clear from its definition that a CGNN defines a continuous FCM. 

\subsubsection{Generative model and interventions} 

A CGNN $\mathcal{C}_{\widehat{\mathcal{G}},\hat{f}} = (\widehat{\mathcal{G}}, \hat{f}, {\cal E})$ is a \textbf{generative} model in the sense that any sample   $[e_{1,j}, \ldots, e_{d,j}]$ of the ``noise'' random vector ${\mathbf E} = [E_1,
\ldots, E_d]$ can be used as ``input'' to the network to generate a data sample $[\hat{x}_{1,j}, \ldots, \hat{x}_{d,j}]$ of the estimated distribution $\hat{P}(\hat{{\mathbf X}} = [\hat{X}_1,\ldots, \hat{X}_d])$ by proceeding as follow:

\begin{enumerate}
\item Draw $\{ [e_{1,j}, \ldots, e_{d,j}] \}_{j=1}^n$, $n$ samples independent identically distributed from the joint distribution of independent noise variables ${\mathbf E} = [E_1, \ldots, E_d]$.
\item Generate $n$ samples $\{ [\hat{x}_{1,j}, \ldots, \hat{x}_{d,j}] \}_{j=1}^{n}$, where each estimate sample $\hat{x}_{i,j}$ of variable $\hat{X}_i$ is computed in the topological order of $\widehat{\mathcal{G}}$ from $\hat{f}_i$ with the $j$-th estimate samples $\hat{x}_{Pa(i;\hat{\mathcal{G}}),j}$ of $\hat{X}_{\text{Pa}(i; \hat{\mathcal{G}})}$ and the $j$-th sample $e_{i,j}$ of the random noise variable $E_i$.
\end{enumerate}

Notice that a CGNN generates a probability distribution $\hat{P}$ which is Markov with respect to $\widehat{\mathcal{G}}$, as the graph  $\widehat{\mathcal{G}}$ is acyclic and the noise variables $E_i$ are mutually independent. 

Importantly, CGNN supports interventions, that is, freezing a variable $X_i$ to some constant $v_i$. The resulting joint distribution noted $\hat{P}_{\text{do}(\hat{X}_i = v_i)}(\hat{X})$, called  \emph{interventional distribution} \citep{pearl2009causality}, can be computed from CGNN by discarding all causal influences on $\hat{X}_i$ and clamping its value to $v_i$.  It is emphasized that intervening is different from conditioning (\emph{correlation does not imply causation}). The knowledge of interventional distributions is essential for e.g., public policy makers, wanting to estimate the overall effects of a decision on a given variable.

\subsection{Model evaluation \label{ModelEval}}

The goal is to associate to each candidate solution $\mathcal{C}_{\widehat{\mathcal{G}},\hat{f}} = (\widehat{\mathcal{G}}, \hat{f}, {\cal E})$ a score reflecting how well this candidate solution describes the observational data. 
Firstly we define the model scoring function  (Section \ref{ModelEval}), then we show that this model scoring function allows to build a CGNN generating a distribution $\hat{P}(\hat{X})$ that approximates $P(X)$ with arbitrary accuracy (Section \ref{Approx}).

\subsubsection{Scoring metric \label{ScoringMetric}}

The ideal score, to be minimized, is the distance between the joint distribution $P$ associated with the ground truth FCM, and the joint distribution $\widehat{P}$ defined by the CGNN candidate $\mathcal{C}_{\hat{\cal G},\hat{f}} = (\hat{\cal G}, \hat{f}, {\cal E})$. A tractable approximation thereof is given by the Maximum Mean Discrepancy (MMD) \citep{gretton2007kernel} between the $n$-sample observational data $\cal D$, and an $n$-sample $\widehat{\cal D}$ sampled after $\widehat{P}$. Overall, the CGNN $\mathcal{C}_{\hat{\cal G},\hat{f}}$ is trained by minimizing   
\begin{equation}
  S(\mathcal{C}_{\hat{\cal G},\hat{f}}, \mathcal{D}) =
   \widehat{\text{MMD}}_k(\mathcal{D}, \widehat{\mathcal{D}}) + \lambda 
    |\widehat{\mathcal{G}}|,
  \label{eq:the_loss}
\end{equation}

with $\widehat{\text{MMD}}_k(\mathcal{D}, \widehat{\mathcal{D}})$ defined as:

\begin{equation}
\widehat{\text{MMD}_k}(\mathcal{D}, \widehat{\mathcal{D}}) = 
\frac{1}{n^2} \sum_{i, j = 1}^{n} k(x_i, x_j) +
\frac{1}{n^2} \sum_{i, j = 1}^{n} k(\hat{x}_i, \hat{x}_j)
- \frac{2}{n^2} \sum_{i,j = 1}^n k(x_i, \hat{x}_j)
\label{eq:mmd}
\end{equation}

\noindent where kernel $k$ usually is taken as the Gaussian kernel ($k(x,x') = \exp(-\gamma \|x-x'\|_2^2)$). The MMD statistic, with quadratic complexity in the sample size, has the good property that as $n$ goes to infinity, it goes to zero iff $P
= \hat{P}$  \citep{gretton2007kernel}. For scalability, a linear approximation of the MMD statistics based on $m=100$ random features \citep{dlp}, 
called $\widehat{\text{MMD}}_k^m$, will also be used in the experiments (more in Appendix~\ref{app:mmd}).

Due to the Gaussian kernel being differentiable, $\widehat{\text{MMD}}_k$ and $\widehat{\text{MMD}}_k^m$ are differentiable, and backpropagation can be used to learn the CGNN made of networks $\hat {f}_i$ structured along $\hat{\cal G}$.

In order to compare candidate solutions with different structures in a fair manner, the evaluation score of Equation \ref{eq:the_loss} is augmented with a penalization term $\lambda |\widehat{\mathcal{G}}|$, 
with $|\widehat{\mathcal{G}}|$ the number of edges in $\hat{\cal G}$. Penalization weight $\lambda$ is a hyper-parameter of the approach.

\subsubsection{Representational power of CGNN \label{Approx}} 

 We note  $\mathcal{D} = \{ [x_{1,j}, \ldots, x_{d,j}]\}_{j=1}^{n}$, the data samples independent  identically distributed after the (unknown) joint  distribution $P({\mathbf X} = [X_1,
\ldots, X_d])$, also referred to as observational data.

Under same conditions as in Proposition \ref{prop1}, ($P(X)$ being decomposable along graph $\cal G$, with continuous and strictly positive joint density function on a compact in $\mathbb{R}^d$ and zero elsewhere), there exists a CGNN $(\widehat{\mathcal {G}}, \hat{f}, {\cal E})$, that approximates $P(X)$ with arbitrary accuracy:

 \begin{prop}{\label{prop2}}
For $m \in [[1,d]]$, let $Z_m$ denote the set of variables with topological order less than $m$ and let $d_m$ be its size. For any $d_m$-dimensional vector of noise values $e^{(m)}$, let $z_m(e^{(m)})$ (resp. $\widehat{z_m}(e^{(m)})$) be the vector of values computed in topological order from the FCM $({\cal G}, f, {\cal E})$ (resp. the CGNN $({\cal G}, \hat{f}, {\cal E})$). 
For any $\epsilon > 0$, there exists a set of networks $\hat{f}$ with architecture $\cal G$ such that 
\begin{equation}
\forall e^{(m)},  \|z_m(e^{(m)})- \widehat{z_m}(e^{(m)})\| < \epsilon
\label{eq:prop2}
\end{equation}
\end{prop}
\begin{proof}
In Appendix \ref{app:proofs}
\end{proof}

Using this proposition and the $\widehat{\text{MMD}_k}$ scoring criterion presented in Equation \ref{eq:mmd}, it is shown that the distribution $\hat P$ of the CGNN can estimate the true observational distribution of the (unknown) FCM up to an arbitrary precision, under the assumption of an infinite observational sample:

 \begin{prop}{\label{prop3}}
Let  $\mathcal{D}$ be an infinite observational sample generated from $({\cal G}, f, {\cal E})$.
With same notations as in Prop. 2, for every  sequence $\epsilon_t$, such that $\epsilon_t>0$ and goes to zero when $t \rightarrow \infty$, there exists a set $\widehat{f_t} = (\hat f^{t}_1 \ldots \hat f^{t}_d)$ such that $\widehat{\text{MMD}_k}$ between $\cal D$ and an infinite size sample $\widehat{\cal D}_{t}$ generated from the CGNN $({\cal G},\widehat{f_{t}},\cal E)$ is less than $\epsilon_t$.
 \end{prop}

\begin{proof}
In Appendix \ref{app:proofs}
\end{proof}

Under these assumptions, as $\widehat{\text{MMD}}_k(\mathcal{D}, \hat{\mathcal{D}_{t}}) \rightarrow 0$, as $t \rightarrow \infty$, it implies that the sequence of generated $\hat{P}_t$ converges in distribution toward  the distribution $P$ of the observed sample \citep{gretton2007kernel}. This result highlights the generality of this approach as we can model any kind of continuous FCM from observational data (assuming access to infinite observational data). Our class of model is not restricted to  simplistic assumptions on the data generative process such as the additivity of the noise or linear causal mechanisms. But this strength comes with a new challenge relative to identifiability of such CGNNs as the result of proposition \ref{prop3} holds for any DAG  $\widehat{\cal G}$ such that $P$ can be factorized along $\cal G$ and then for any any DAG in the Markov equivalence class of ${\cal G}$ (under classical assumption of CMA, CFA and CSA).  In particular in the pairwise setting, when only 2 variables $X$ and $Y$ are observed, the joint distribution $P(X,Y)$ can be factorized in two Markov equivalent DAGs $X\rightarrow Y$ or $Y\rightarrow X$ as $P(X,Y) = P(X) P(Y|X)$ and $P(X,Y) = P(Y) P(X|Y)$. Then the CGNN can reproduce equally well the observational distribution in both directions (under the assumption of proposition \ref{prop1}).  We refer the reader to \cite{zhang2009identifiability} for more details on this problem of identifiability in the  bivariate case. 

As shown in Section \ref{sec:identif}, the proposed approach enforces the discovery of causal models in the Markov equivalence class. Within this class, the non-identifiability issue is empirically mitigated by restricting the class of CGNNs considered, and specifically limiting the number $n_h$ of hidden neurons in each causal mechanism (Eq. \ref{eq:fcm_neural}). Formally, we restrict ourselves to the sub-class of CGNNs, noted $\mathcal{C}_{\hat{\cal G},\hat{f}^{n_h}} = (\widehat{\mathcal{G}}, \hat{f}^{n_h}, \mathcal{E})$ with exactly $n_h$ hidden neurons in each  $\hat{f}_i$ mechanism. Accordingly, any candidate  $\hat{\mathcal{G}}$ with number of edges $|\hat{\mathcal{G}}|$ involves the same number of parameters: $(2d+|\hat{\mathcal{G}}|) \times n_h$ weights and $d \times (n_h + 1 )$ bias parameters. As shown experimentally in Section \ref{sec:expe}, this parameter $n_h$ is crucial as it governs the CGNN ability to model the causal mechanisms: too small $n_h$, and data patterns may be missed; too large $n_h$, and overly complicated causal mechanisms may be retained.

\subsection{Model optimization \label{sec:optim}}

 Model optimization consists at finding a (nearly) optimum solution $(\hat{\mathcal{G}}, \hat{f})$ in the sense of the score defined in the previous section. 
  The so-called {\em parametric} optimization of the CGNN, where structure estimate $\hat{\mathcal{G}}$ is fixed and the goal is to find the best neural estimates $\hat{f}$ conditionally to $\hat{\mathcal{G}}$ is tackled in Section \ref{sec:param}. The {\em non-parametric} optimization, aimed at finding the best structure estimate, is considered in Section \ref{sec:nonparam}. In Section \ref{sec:identif}, we present an identifiability result for CGNN up to Markov equivalence classes.

\subsubsection{Parametric (weight) optimization \label{sec:param}}

 Given the acyclic structure estimate $\hat{\mathcal{G}}$, the neural networks $\hat{f}_1, \ldots, \hat{f}_d$ of the CGNN are learned end-to-end using backpropagation with Adam optimizer \citep{2014arXiv1412.6980K} by minimizing losses $\widehat{\text{MMD}_k}$ (Eq. \eqref{eq:mmd}, referred to as \textbf{CGNN} ($\widehat{\text{MMD}}_k$) ) or $\widehat{\text{MMD}}_k^m$ (see Appendix \ref{app:mmd},\textbf{CGNN} ($\widehat{\text{MMD}}^m_k$)).

The procedure closely follows that of supervised continuous learning (regression), except for the fact that the loss to be minimized is the MMD loss instead of the mean squared error.
Neural nets $\hat{f}_i$, $i \in [[1,d]]$ are trained during $n_\text{train}$ epochs, where the noise samples, independent and identically distributed, are drawn in each epoch. In the $\widehat{\text{MMD}}_k^m$ variant, the parameters of the random kernel are resampled from their respective distributions in each training epoch (see Appendix \ref{app:mmd}). After training, the score is computed and averaged over $n_{eval}$ estimated samples of size $n$. Likewise, the noise samples are re-sampled anew for each evaluation sample. The overall process with training and evaluation is repeated $nb_{run}$ times to reduce stochastic effects relative to random initialization of neural network weights and stochastic gradient descent.

\subsubsection{Non-parametric (structure) optimization}  \label{search_structure}\label{ModelSearch} \label{sec:nonparam}

The number of directed acyclic graphs $\mathcal{\hat G}$ over $d$ nodes is super-exponential in $d$, making the non-parametric optimization of the CGNN structure an intractable computational and statistical problem. 
Taking inspiration from \cite{tsamardinos2006max,nandy2015high}, we start from a graph skeleton recovered by other methods such as  feature selection \citep{yamada2014high}. We  focus on optimizing the edge orientations. Letting $L$ denote the number of edges in the graph, it defines a combinatorial optimization problem of  complexity ${\cal O}(2^L)$ (note however that not all orientations are admissible since the eventual oriented graph must be a DAG). 

The motivation for this approach is to decouple the edge selection task and the causal modeling (edge orientation) tasks, and enable their independent assessment.

Any $X_i-X_j$ edge in the graph skeleton stands for a direct dependency between variables $X_i$ and $X_j$.
Given Causal Markov and Faithfulness assumptions, such a direct dependency either reflects a direct causal relationship between the two variables ($X_i \rightarrow X_j$ or $X_i \leftarrow X_j$), or is due to the fact that $X_i$ and $X_j$ admit a latent (unknown) common cause ($X_i \leftrightarrow X_j$). Under the assumption of \textit{causal sufficiency}, the latter does not hold. Therefore the $X_i-X_j$ link is associated with a causal relationship in one or the other direction. The causal sufficiency assumption will be relaxed in Section \ref{confounder}.

The edge orientation phase proceeds as follows:

\begin{itemize}
\item[$\bullet$] Each  $X_i-X_j$ edge is first considered in isolation, and its orientation is evaluated using CGNN. Both score $S(\mathcal{C}_{X_i \rightarrow X_j,\hat{f}}, \mathcal{D}_{ij})$ and $S(\mathcal{C}_{X_j \rightarrow X_i,\hat{f}}, \mathcal{D}_{ij})$ are computed, where $\mathcal{D}_{ij} = \{[x_{i,l}, x_{j,l}]\}_{l=1}^{n}$. The best orientation corresponding to a minimum score is retained.  After this step, an initial graph is built with complexity $2L$ with $L$ the number of edges in the skeleton graph. 
\item[$\bullet$] The initial graph is revised to remove all cycles. Starting from a set of random nodes, all paths are followed iteratively until all nodes are reached; an edge pointing toward an already visited node and forming a cycle is reversed. The resulting DAG is used as initial DAG for the structured optimization, below. 
\item[$\bullet$] The optimization of the DAG structure is achieved using a hill-climbing algorithm aimed to optimize  the global score $ S(\mathcal{C}_{\hat{\cal G},\hat{f}}, \mathcal{D})$. Iteratively, i) an edge $X_i-X_j$ is uniformly randomly selected in the current graph; ii) the graph obtained by reversing this edge is considered (if it is still a DAG and has not been considered before) and the associated global CGNN is retrained; iii) if this graph obtains a lower global score than the former one, it becomes the current graph and the process is iterated until reaching a (local) optimum. More sophisticated combinatorial optimization approaches, e.g. Tabu search, will be considered in further work. In this paper, hill-climbing is used for a proof of concept of the proposed approach, achieving a decent trade-off  between computational time and accuracy.
\end{itemize}

At the end of the process each causal  edge $X_i \rightarrow X_j$ in $\mathcal{G}$ is associated with a score, measuring its contribution to the global score:

\begin{equation}
S_{X_i \rightarrow X_j} = S(\mathcal{C}_{\hat{\mathcal{G}} - \{X_i \rightarrow X_j\},\hat{f}}, \mathcal{D}) - S(\mathcal{C}_{\hat{\cal G},\hat{f}}, \mathcal{D}) 
\label{score_edges}
\end{equation}

During the structure (non-parametric) optimization, the graph skeleton is fixed; no edge is added or removed. The penalization term $\lambda |\widehat{\mathcal{G}}|$ entering in the score evaluation (eq. \ref{eq:the_loss}) can thus be neglected at this stage and only the MMD-losses are used to compare two graphs. The penalization term will be used in Section \ref{confounder} to compare structures with different skeletons, as the potential confounding factors will be dealt with by removing edges.

\subsubsection{ Identifiability of CGNN up to Markov equivalence classes \label{sec:identif}}

Assuming an infinite number of observational data, and assuming further that the generative distribution belongs to the CGNN class $\mathcal{C}_{\mathcal{G},f}$, then there exists a DAG reaching an MMD score of 0 in the Markov equivalence class of $\mathcal{G}$:  

\begin{prop}\label{prop4}
Let $\textbf{X} = [X_1, \ldots, X_d]$ denote a set of continuous random variables with joint distribution $P$, generated by a CGNN $\mathcal{C}_{\mathcal{G},f} = (\mathcal{G}, f, \mathcal{E})$ with ${\cal G}$ a directed acyclic graph. Let $\mathcal{D}$ be an infinite observational sample generated from this CGNN. We assume that $P$ is Markov and faithful to the graph ${\cal G}$, and that every  pair of variables $(X_i,X_j)$ that are  d-connected in the graph are not independent. We note $\widehat{\mathcal{D}}$ an infinite sample generated by a candidate CGNN, $\mathcal{C}_{\widehat{\mathcal{G}},\hat{f}} = (\widehat{\mathcal{G}}, \hat{f}, \mathcal{E})$. Then, \\
(i)\ \ If $\widehat{\mathcal{G}} = {\cal G}$ and $\hat{f} = f$, then $\widehat{\text{MMD}}_k(\mathcal{D}, \widehat{\mathcal{D}}) = 0$.\\
(ii) For any  graph $\widehat{\mathcal{G}}$ characterized by the same adjacencies but not belonging to the Markov equivalence class of $\mathcal{G}$, for all $\hat{f}$, $\widehat{\text{MMD}}_k(\mathcal{D}, \widehat{\mathcal{D}}) \neq 0$.
\end{prop}

\begin{proof} 
In Appendix \ref{app:proofs}
\end{proof}

This result does not establish the  CGNN identifiability within the Markov class of equivalence, that is left for future work. As shown experimentally in Section \ref{expe_setting}, there is a need to control the model capacity in order to recover the directed graph in the Markov equivalence class.\footnote{In some specific cases, such as in the bivariate linear FCM with Gaussian noise and Gaussian input, even by restricting the class of functions considered, the DAG cannot be identified from purely observational data \citep{mooij2016distinguishing}.}

\section{Experiments}\label{sec:expe}

This section reports on the empirical validation of CGNN compared to the state of the art under the no confounding assumption. The experimental setting is first discussed. Thereafter, the results obtained in the bivariate case, where only asymmetries in the joint distribution can be used to infer the causal relationship, are discussed. The variable triplet case, where conditional independence can be used to uncover causal orientations, and the general case of $d > 2$ variables are finally considered. 
All computational times are measured on Intel Xeon 2.7Ghz (CPU) or on Nvidia GTX 1080Ti graphics card (GPU).

\subsection{Experimental setting \label{expe_setting}}

The CGNN architecture is a 1-hidden layer network with ReLU activation function. The multi-scale Gaussian kernel used in the MMD scores has bandwidth $\gamma$ ranging in $\{0.005, 0.05, 0.25,0.5, 1,5,50\}$. The number $nb_{run}$ used to average the score is set to 32 for CGNN-MMD (respectively 64 for CGNN-Fourier).  In this section the distribution $\mathcal{E}$ of the noise variables is set to $\mathcal{N}(0,1)$. 
The number $n_h$ of neurons in the hidden layer, controlling the identifiability of the model, is the most sensitive hyper-parameter of the presented approach. Preliminary experiments are conducted to adjust its range, as follows. A 1,500 sample dataset is generated from the linear structural equation model with additive uniform noise $Y=X+\mathcal{U}(0, 0.5), X \sim U([-2,2])$ (Fig. \ref{figure:parallelogram}). Both CGNNs associated to $X \rightarrow Y$ and $Y \rightarrow X$ are trained until reaching convergence ($n_{epoch} = 1,000$) using Adam \citep{2014arXiv1412.6980K} with a learning rate of $0.01$ and evaluated over $n_{eval} = 500$ generated samples. The distributions generated from both generative models are displayed on Fig. \ref{figure:parallelogram} for $n_h = 2, 5, 20, 100$. The associated scores (averaged on 32 runs) are  displayed on Fig. \ref{figure:parallelogramme-pareto}, confirming that the model space must be restricted for the sake of identifiability (cf. Section \ref{sec:identif} above).

\begin{figure}[h!]
\begin{center}
\includegraphics[width=12cm]{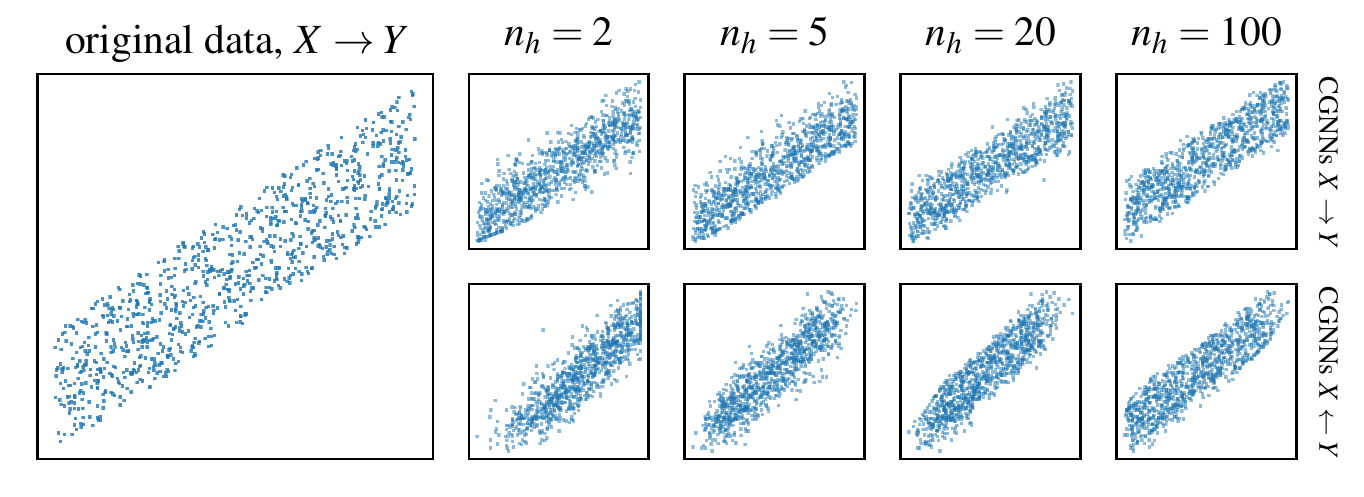}
\caption{Leftmost: Data samples. Columns 2 to 5: Estimate samples generated from CGNN with direction $X \rightarrow Y$ (top row)  and $Y \rightarrow X$ (bottom row) for number of hidden neurons $n_h = 2, 5, 20, 100$.}
\label{figure:parallelogram}
\end{center}
\end{figure} 

\begin{figure}[!h]
       \begin{subfigure}{0.59\textwidth}
       \centering
       \includegraphics[width=6cm]{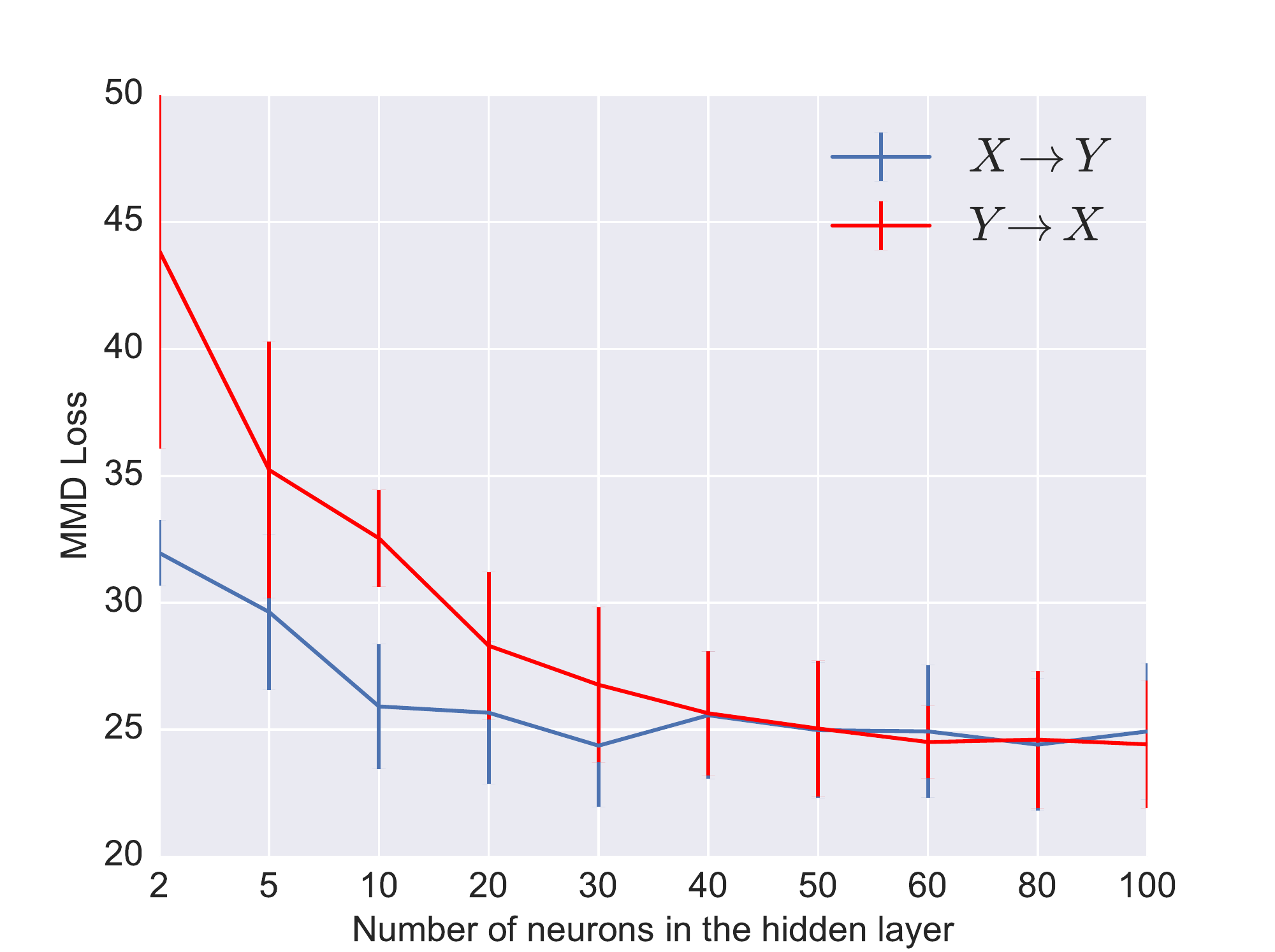}
       \caption{$C_{X\rightarrow Y}$, $C_{Y\rightarrow X}$ with various $n_h$ values.}
       \label{figure:parallelogramme-pareto}
   \end{subfigure}
   \begin{subtable}{0.39\textwidth}
         \footnotesize
         \caption{Scores  $C_{X\rightarrow Y}$ and $C_{Y\rightarrow X}$ with their difference. $^{\star \star \star }$ denotes the significance at the  0.001 threshold with the t-test.}
         \label{table:diff_scores_parallelogram}
           \begin{tabular}{llll}
           \toprule
            $n_h$ & $C_{X\rightarrow Y}$ & $C_{Y\rightarrow X}$ & Diff.\\
           \midrule
           2 & $32.0$ & $43.9$ & $11.9^{\star \star \star }$ \\
           5 & $29.6$ & $35.2$ & $5.6^{\star \star \star }$ \\
           10 & $25.9$ & $32.5$ & $6.6^{\star \star \star }$ \\
           20 & $25.7$ & $28.3$ & $2.6^{\star \star \star }$ \\
           30 & $24.4$ & $26.8$ & $2.4^{\star \star \star }$ \\
           40 &  $25.6$ & $25.6$ & $0.7$ \\
           50 &  $25.0$ & $25.0$ & $0.6$ \\
           100 & $24.9$ & $24.4$ & $-0.5$ \\
           \bottomrule
      \end{tabular} 
    \end{subtable}
    \caption{CGNN sensitivity w.r.t. the number of hidden neurons $n_h$: Scores associated to both causal models (average and standard deviation over 32 runs).}
\end{figure}

\subsection{Learning bivariate causal structures \label{bivariate_structure}}\label{sec:bivariate}

As said, under the no-confounder assumption a dependency between variables $X$ and $Y$ exists iff either $X$ causes $Y$ ($Y = f(X,E)$) or $Y$ causes $X$ ($X = f(Y,E)$). The identification of a \textit{Bivariate Structural Causal Model}  is based on comparing the model scores (Section \ref{ModelEval}) attached to both CGNNs.

\paragraph{Benchmarks.} Five datasets with continuous variables are considered:\footnote{The first four datasets are available at \url{http://dx.doi.org/10.7910/DVN/3757KX}. The \textit{Tuebingen cause-effect pairs} dataset is available at \url{https://webdav.tuebingen.mpg.de/cause-effect/}}\\
$\bullet$ \textbf{CE-Cha}: 300 continuous variable pairs from the cause effect pair challenge \citep{guyon2013cepc}, restricted to pairs with label  $+1$ ($X \rightarrow Y$) and $-1$ ($Y \rightarrow X$). \\
$\bullet$ \textbf{CE-Net}: 300 artificial pairs generated with a neural network initialized with random weights and random distribution for the cause (exponential, gamma, lognormal, laplace...).\\
$\bullet$ \textbf{CE-Gauss}: 300 artificial pairs without confounder sampled with the generator of  \cite{mooij2016distinguishing}:
$Y = f_Y(X,E_Y)$ and $X = f_X(E_X)$ with  $E_X \sim p_{E_X}$ and  $E_Y \sim p_{E_Y}$. $p_{E_X}$ and $p_{E_Y}$ are randomly generated Gaussian mixture distributions. Causal mechanism $f_X$ and $f_Y$ are randomly generated Gaussian processes. \\
$\bullet$ \textbf{CE-Multi}: 300 artificial pairs generated with linear and polynomial mechanisms. The effect variables are built with post additive noise setting ($Y = f(X) + E$), post multiplicative noise ($Y = f(X) \times E$), pre-additive noise ($Y = f(X + E)$) or pre-multiplicative noise ($Y = f(X \times E)$).\\
$\bullet$ \textbf{CE-Tueb}: 99 real-world cause-effect pairs from the \textit{Tuebingen cause-effect pairs} dataset, version August 2016 \citep{mooij2016distinguishing}. This version of this dataset is taken from 37 different data sets coming from various domain: climate, census, medicine data.

For all variable pairs, the size $n$ of the data sample is set to 1,500 for the sake of an acceptable overall computational load. 

\paragraph{Baseline approaches.}
CGNN is assessed comparatively to the following algorithms:\footnote{Using the  R program available at \url{https://github.com/ssamot/causality} for ANM, IGCI, PNL, GPI and LiNGAM.} i) ANM \citep{mooij2016distinguishing} with Gaussian process regression and HSIC independence test of the residual; ii) a pairwise version of LiNGAM \citep{shimizu2006linear} relying on Independent Component Analysis to identify the linear relations between variables; iii) IGCI \citep{daniusis2012inferring} with entropy estimator and Gaussian reference measure; iv) the post-nonlinear model (PNL) with HSIC test \citep{zhang2009identifiability}; v) GPI-MML \citep{stegle2010probabilistic}; where the Gaussian process regression with
higher marginal likelihood is selected as causal direction; vi) CDS, retaining the causal orientation with lowest variance of the conditional probability distribution; vii) Jarfo \citep{fonollosa2016conditional}, using a random forest causal classifier trained from the ChaLearn Cause-effect pairs on top of 150 features including ANM, IGCI, CDS, LiNGAM, regressions, HSIC tests.

\paragraph{Hyper-parameter selection.}
For a fair comparison, a leave-one-dataset-out procedure is used to select the key best hyper-parameter for each algorithm. To avoid computational explosion, a single hyper-parameter per algorithm is adjusted in this way; other hyper-parameters are set to their default value. 
For CGNN, $n_h$ ranges over $\{5,\ldots,100\}$. The leave-one-dataset-out procedure sets this hyper-parameter $n_h$ to values between 20 and 40 for the different datasets. 
For ANM and the bivariate fit, the kernel parameter for the Gaussian process regression ranges over $\{0.01, \ldots, 10\}$. For PNL, the threshold parameter alpha for the HSIC independence test ranges over $\{0.0005, \ldots,0.5\}$. For CDS, the $ffactor$ involved in the discretization step ranges over $[[1,10]]$. For GPI-MML, its many parameters are set to their default value as none of them appears to be more critical than others. 
Jarfo is trained from  4,000 variable pairs datasets with same generator used for  $\textbf{CE-Cha-train}$, \textbf{CE-Net-train}, \textbf{CE-Gauss-train} and \textbf{CE-Multi-train}; the causal classifier is trained on all datasets except the test set.

\paragraph{Empirical results.} Figure \ref{fig:score_pairwise} reports the area under the precision/recall curve for each benchmark and all algorithms. 

\begin{figure}[h!]
\centering
\includegraphics[width=\textwidth]{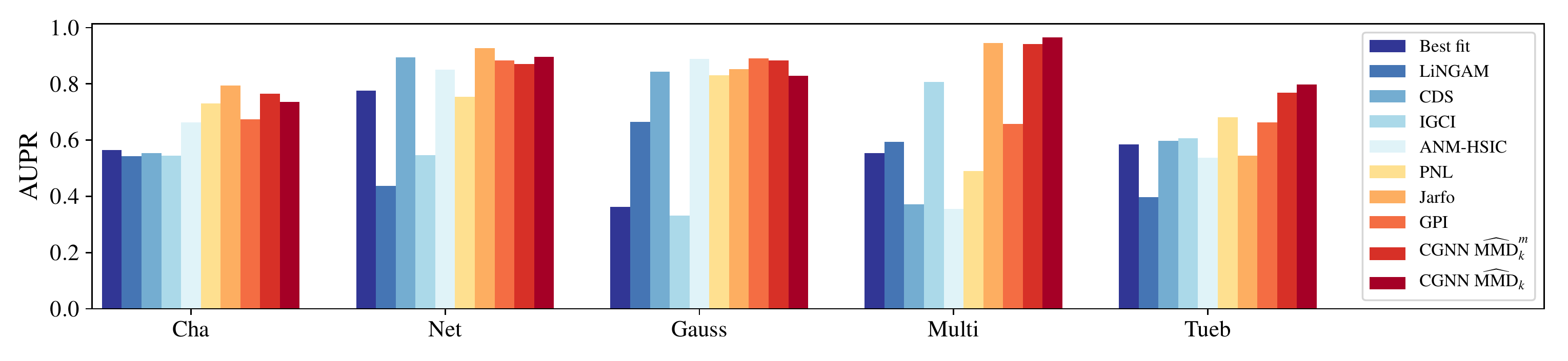}
\caption{Bivariate Causal Modelling: Area under the precision/recall curve for the five datasets. A full table of the scores is given in Appendix \ref{TableResultsPairwise}.}
\label{fig:score_pairwise}
\end{figure}

Methods based on simple regression like the bivariate fit and Lingam are outperformed as they underfit the data generative process. CDS and IGCI obtain very good results on few datasets. 
Typically, IGCI takes advantage of some specific features of the dataset, (e.g. the cause entropy being lower than the effect entropy  in \textbf{CE-Multi}), but remains at chance level otherwise.   ANM-HSIC yields good results when the additive assumption holds (e.g. on $ \textbf{CE-Gauss}$), but fails otherwise. PNL, less restrictive than ANM, yields overall good  results compared to the former methods.  Jarfo, a voting procedure, can in principle yield the best of the above methods and does obtain good results on artificial data. However, it does not perform well on the real dataset \textbf{CE-Tueb}; this counter-performance is blamed on the differences between all five benchmark distributions and the lack of generalization / transfer learning.

Lastly, generative methods {GPI} and \textbf{CGNN} ($\widehat{\text{MMD}}_k$) perform well on most datasets, including the real-world cause-effect pairs
{CE-T\"ub}, in counterpart for a higher computational cost (resp. 32 min on CPU for GPI and 24 min on GPU for CGNN). Using the linear MMD approximation \cite{dlp}, \textbf{CGNN} ($\widehat{\text{MMD}}^m_k$ as explained in appendix \ref{app:mmd})  reduces the cost by a factor of 5 without hindering the performance. 

Overall, CGNN demonstrates competitive performance on the cause-effect inference problem, where it is necessary to discover distributional asymmetries.

\subsection{Identifying v-structures }

A second series of experiments is conducted to investigate the method performances on variable triplets, where  multivariate effects and conditional variable independence must be taken into account to identify the Markov equivalence class of a DAG. 
The considered setting is that of variable triplets $(A, B, C)$ in the linear Gaussian case, where asymmetries between cause and effect cannot be exploited \citep{shimizu2006linear} and conditional independence tests are required. In particular strict pairwise methods can hardly be used due to un-identifiability (as each pair involves a linear mechanism with Gaussian input and additive Gaussian noise)  \citep{hoyer2009nonlinear}.

With no loss of generality, the graph skeleton involving variables $(A, B, C)$ is $A-B-C$. All three causal models (up to variable renaming) based on this skeleton are used to generate 500-sample datasets, where the random noise variables are independent centered Gaussian variables.

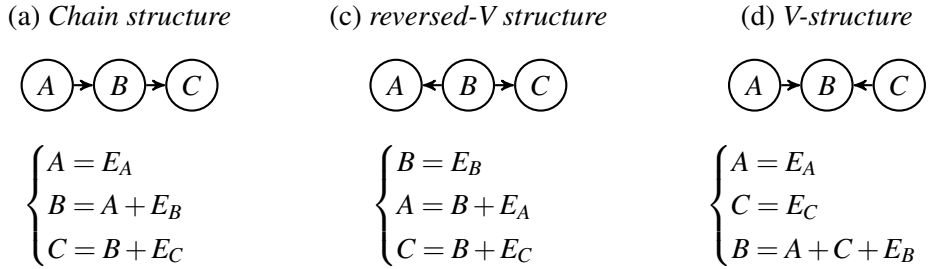
\begin{figure}[!h]
    \centering
    \begin{tikzpicture}[node distance=0.25cm, auto,]
        \node[punkt] (A) at (0,0)  {$A$};
        \node[punkt, right=of A] (B) {$B$};
        \node[punkt, right=of B] (C) {$C$};
        \node[above=of B] (cap) {(a) \textit{Chain structure}};
        \node[below=of B] (cap) {$\begin{cases} A = E_A\\
                                 B= A + E_B\\
                                 C= B + E_C\\
                                 \end{cases}$}; 
        \draw[pil] (A) -- (B);
        \draw[pil] (B) -- (C);
    \end{tikzpicture}
    \hspace{2pc}
    \begin{tikzpicture}[node distance=0.25cm, auto,]
        \node[punkt] (A) at (0,0)  {$A$};
        \node[punkt, right=of A] (B) {$B$};
        \node[punkt, right=of B] (C) {$C$};
        \node[above=of B] (cap) {(c) \textit{reversed-V structure}};
        \node[below=of B] (cap) {$\begin{cases} B= E_B\\
                 A = B + E_A\\
                 C = B + E_C\\
                 \end{cases}$}; 
        \draw[pil] (B) -- (A);
        \draw[pil] (B) -- (C);
    \end{tikzpicture}
    \hspace{2pc}
    \begin{tikzpicture}[node distance=0.25cm, auto,]
        \node[punkt] (A) at (0,0)  {$A$};
        \node[punkt, right=of A] (B) {$B$};
        \node[punkt, right=of B] (C) {$C$};
        \node[above=of B] (cap) {(d) \textit{V-structure}};
        \node[below=of B] (cap) {$\begin{cases} A = E_A\\
         C = E_C\\
         B = A + C + E_B\\
         \end{cases}$}; 
        \draw[pil] (A) -- (B);
        \draw[pil] (C) -- (B);
    \end{tikzpicture}
    \caption{Datasets generated from the three DAG configurations with skeleton $A-B-C$ }
    \label{DAG3variables}
\end{figure}

Given skeleton $A-B-C$, each dataset is used to model the 
possible four CGNN structures (Fig. \ref{DAG3variables}, with generative SEMs): 

\begin{itemize}
\item Chain structures $ABC$ ($A = f_1(E_1)$, $B = f_2(A,E_2)$ , $C = f_3(B,E_3)$ and $CBA$ ($C = f_1(E_1)$, $B = f_2(C,E_2)$ , $A = f_3(B,E_3)$)
\item V structure: $A = f_1(E_1)$, $C = f_2(E_2)$ , $B = f_3(A,C,E_3)$
\item reversed V structure: $B = f_1(E_1)$, $A = f_2(B,E_2)$ , $C = f_3(B,E_3)$
\end{itemize}

Let $C_{ABC}$, $C_{CBA}$, 
$C_{V-structure}$ and $C_{reversed V}$ denote the scores of the CGNN models respectively attached to these structures. The scores computed on all three  datasets are displayed in Table \ref{resultMarkov} (average over 64 runs; the standard deviation is indicated in parenthesis).

\begin{table}[!h]
 
  \label{table:acc_graph}
  \centering
  \begin{tabular}{l|cc|ccc}
    \toprule

     & \multicolumn{2}{c|}{non V-structures} & V structure \\
   Score & Chain str. & 
    Reversed-V str. & V-structure \\
    \midrule
    $C_{ABC}$ & \textit{0.122 (0.009)} & 
    \textit{0.124 (0.007)} &  0.172 (0.005) \\
    $C_{CBA}$ & \textit{0.121 (0.006)} & 
    \textit{0.127 (0.008)} & 0.171 (0.004)  \\
    $C_{reversed V}$ & \textit{0.122 (0.007)} & 
    \textit{0.125 (0.006)} &  0.172 (0.004) \\ \hline
    $C_{Vstructure}$ & 0.202 (0.004) & 
    0.180 (0.005)   & \textbf{0.127} (0.005) \\
  \end{tabular}
  \caption{CGNN-MMD scores for all models on all datasets. Smaller scores indicate a better match. CGNN correctly identifies V-structure vs. other structures.}
  \label{resultMarkov}
\end{table}

CGNN scores support a clear and significant discrimination between the V-structure and all other structures (noting that the other structures are Markov equivalent and thus can hardly be distinguished). 

This second series of experiments thus shows that CGNN can effectively detect, and take advantage of, conditional independence between variables. 

\subsection{Multivariate causal modeling under Causal Sufficiency Assumption \label{LearningCausalMarkov}}\label{sec:multi}

Let ${\mathbf X} = [X_1,...,X_d]$ be a set of continuous variables, satisfying the Causal Markov, faithfulness and causal sufficiency assumptions. To that end, all experiments provide all algorithms {\em the true graph skeleton}, so their ability to orient edges is compared in a fair way. This allows us to separate the task of orienting the graph from that of uncovering the skeleton.

\subsubsection{Results on artificial graphs with additive and multiplicative noises} 

We draw $500$ samples from $20$ training artificial causal graphs and $20$ test artificial causal graphs on 20 variables. Each variable has a number of parents uniformly drawn in $[[0,5]]$; $f_i$s are randomly generated polynomials involving additive/multiplicative noise.\footnote{The data generator is available at \url{https://github.com/GoudetOlivie/CGNN}. The datasets considered are available at \url{http://dx.doi.org/10.7910/DVN/UZMB69}.}

We compare CGNN to the PC algorithm \cite{spirtes1993search}, the score-based methods GES \cite{chickering2002optimal}, LiNGAM  \cite{shimizu2006linear}, causal additive model (CAM) \cite{buhlmann2014cam}
and with the pairwise methods ANM and Jarfo. For PC, we employ the better-performing, order-independent version of the PC algorithm proposed by \cite{colombo2014order}. PC needs the specification of a conditional independence test. We compare PC-Gaussian, which
employs a Gaussian conditional independence test on Fisher z-transformations,
and PC-HSIC, which uses the HSIC conditional independence test with the Gamma approximation \cite{gretton2005kernel}. PC and GES are implemented in the \textit{pcalg} package \cite{kalisch2012causal}.

All hyperparameters are set on the training graphs in order to maximize the Area Under the Precision/Recall score (AUPR). For the Gaussian conditional independence test and the HSIC conditional independence test, the significance level achieving best result on the training set are respectively $0.1$ and $0.05$ .  For GES, the penalization parameter is set to $3$ on the training set.  For CGNN, $n_h$ is set to 20 on the training set. For CAM, the cutoff value is set to $0.001$.

Figure \ref{fig:score_graph} (left) displays the performance of all algorithms obtained by starting from the exact skeleton on the test set of artificial graphs and  measured from the AUPR (Area Under the Precision/Recall curve), the Structural Hamming Distance (SHD, the number of edge modifications to transform one graph into another)
and the Structural Intervention Distance (SID, the number of equivalent two-variable interventions between two graphs)
\cite{peters2013structural}.

\begin{figure}[h!]
\centering
\includegraphics[width=\textwidth]{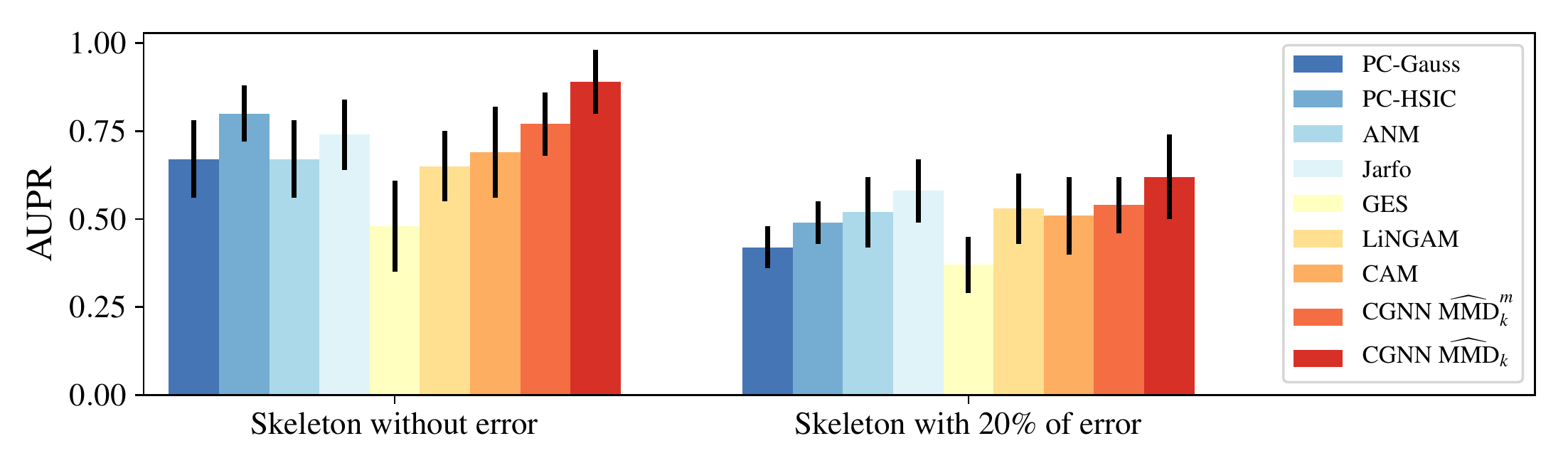}
\caption{Average (std. dev.) AUPR results for the orientation of 20 artificial graphs given true skeleton (left) and artificial graphs given skeleton with 20\% error (right). A full table of the scores, including the metrics Structural Hamming Distance (SHD) and Structural Intervention (SID) \citep{peters2013structural} is shown on Table 4 in section in section "Table of Scores for the Experiments on Graphs" in Appendix. }
\label{fig:score_graph}
\end{figure}

CGNN obtains significant better results with SHD and SID  compared to the other algorithms when the task is to discover the causal from the true skeleton. One resulting graph is shown on Figure \ref{fig:graph_multadd_cgnn}. There are 3 mistakes on this graph (red edges) (in lines with an SHD on average of 2.5).

Constraints based method PC with powerful HSIC conditional independence test is the second best performing method. It highlights the fact that when the skeleton is known, exploiting the structure of the graph leads to good results compared to pairwise methods using only local information. Notably, as seen on Figure \ref{fig:graph_multadd_cgnn}, this type of DAG has a lot of v-structure, as many nodes have more than one parent in the graph, but this is not always the case as shown in the next subsection.

Overall CGNN and PC-HSIC are the most computationally expensive methods, taking an average of 4 hours on GPU and 15 hours on CPU, respectively.

The robustness of the approach is validated by randomly perturbing 20\% edges in the graph skeletons provided to all algorithms (introducing about 10 false edges over 50 in each skeleton). As shown on Table 4 (right) in Appendix, and as could be expected, the scores of all algorithms are lower when spurious edges are introduced. Among the least robust methods are constraint-based methods; a tentative explanation is that they heavily rely on the graph structure to orient edges. By comparison pairwise methods are more robust because each edge is oriented separately. As CGNN leverages conditional independence but also distributional asymmetry like pairwise methods, it obtains overall more robust results when there are errors in the skeleton compared to PC-HSIC. However one can notice that a better SHD score is obtained by CAM, on the skeleton with 20\% error. This is due to the exclusive last edge pruning step of CAM, which removes spurious links in the skeleton.

CGNN obtains overall good results on these artificial datasets. It offers the advantage to  deliver a full generative model useful for simulation (while e.g., Jarfo and PC-HSIC only give the causality graph). 
 To explore the scalability of the approach, 5 artificial graphs with $100$ variables have been considered, achieving an AUPRC of $85.5 \pm 4$, in 30 hours of computation on four NVIDIA 1080Ti GPUs.  

\begin{figure}[h!]
\centering
\includegraphics[width=0.35\textwidth]{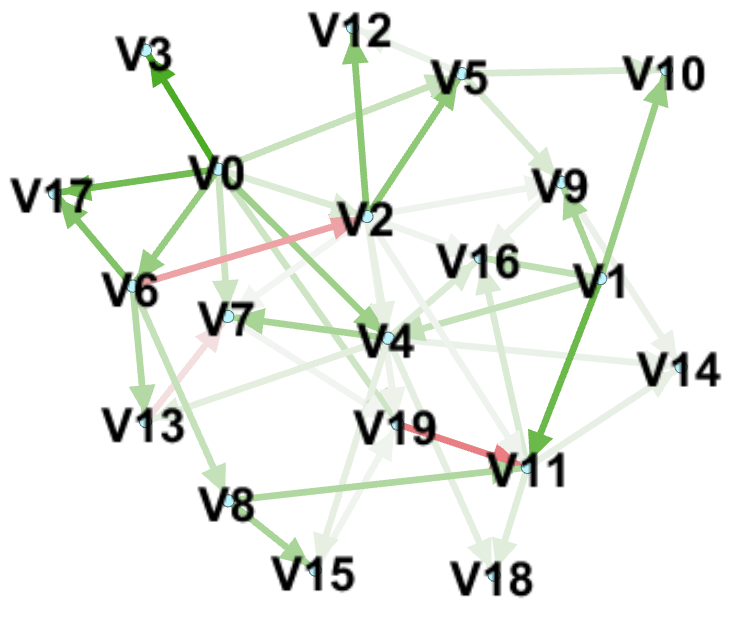}
\caption{Orientation by CGNN of artificial graph with 20 nodes. Green edges are good orientation and red arrows false orientation. 3 edges are red and 42 are green. The strength of the line refers to the confidence of the algorithm.}
\label{fig:graph_multadd_cgnn}
\end{figure}

\subsubsection{Result on biological data}

We now evaluate CGNN on biological networks. First we apply it on simulated gene expression data and then on real protein network.

\paragraph{Syntren artificial simulator}

 First we apply CGNN on SynTREN \citep{van2006syntren} from sub-networks of E. coli \citep{shen2002network}.
 SynTREN creates synthetic transcriptional regulatory networks and produces simulated gene expression data that approximates experimental data. Interaction kinetics are modeled by complex mechanisms based on Michaelis-Menten and Hill kinetics \citep{mendes2003artificial}.
 
With Syntren, we simulate 20 subnetworks of 20 nodes and 5 subnetworks with 50 nodes. For the sake of reproducibility, we use the random seeds  of $0,1 \ldots 19$ and $0,1 \ldots 4$ for each graph generation with respectively 20 nodes and 50 nodes.  The default Syntren parameters are used: a probability of 0.3 for complex 2-regulator interactions and a value of 0.1 for Biological noise, experimental noise and Noise on correlated inputs. For each graph, Syntren give us expression datasets with 500 samples. 

\begin{figure}[h!]
\centering
\includegraphics[width=\textwidth]{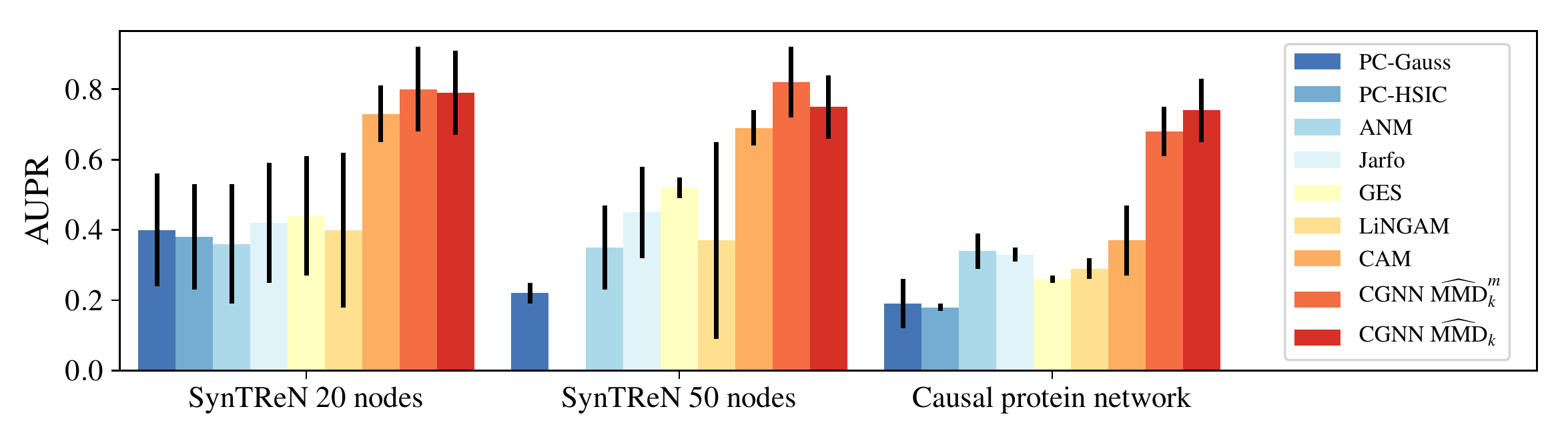}
\caption{Average (std. dev.) AUPR results for the orientation of 20 artificial graphs generated with the SynTReN simulator with 20 nodes (left), 50 nodes (middle), and real protein network given true skeleton (right). A full table of the scores, including the metrics Structural Hamming Distance (SHD) and Structural Intervention (SID) \citep{peters2013structural} is included in section "Table of Scores for the Experiments on Graphs" in Appendix.}
\label{fig:score_bio}
\end{figure}

Figure \ref{fig:score_bio} (left and middle) and Table \ref{table:syntren_network} in section "Table of Scores for the Experiments on Graphs" in Appendix display the performance of all algorithms obtained by starting from the exact skeleton of the causal graph with same hyper-parameters as in the previous subsection. As a note, we canceled the PC-HSIC algorithm after 50 hours of running time.

 Constraint based methods obtain low score on this type of graph dataset. It may be explained by the type of structure involved. Indeed as seen of Figure \ref{fig:graph_syntren_cgnn}, there are very few v-structures in this type of network, making impossible the orientation of an important number of edges by using only conditional independence tests. Overall the methods CAM and CGNN that take into account of both distributional asymmetry and multivariate interactions, get the best scores. CGNN obtain the best results in AUPR, SHD and SID for graph with 20 nodes and 50 nodes, showing that this method can be used to infer networks having complex distribution, complex causal mechanisms and interactions. The Figure \ref{fig:graph_syntren_cgnn} shows the resulting graph obtain with CGNN. Edges with good orientation are displayed in green and edge with false orientation in red.

\begin{figure}[h!]
\centering
\includegraphics[width=\textwidth]{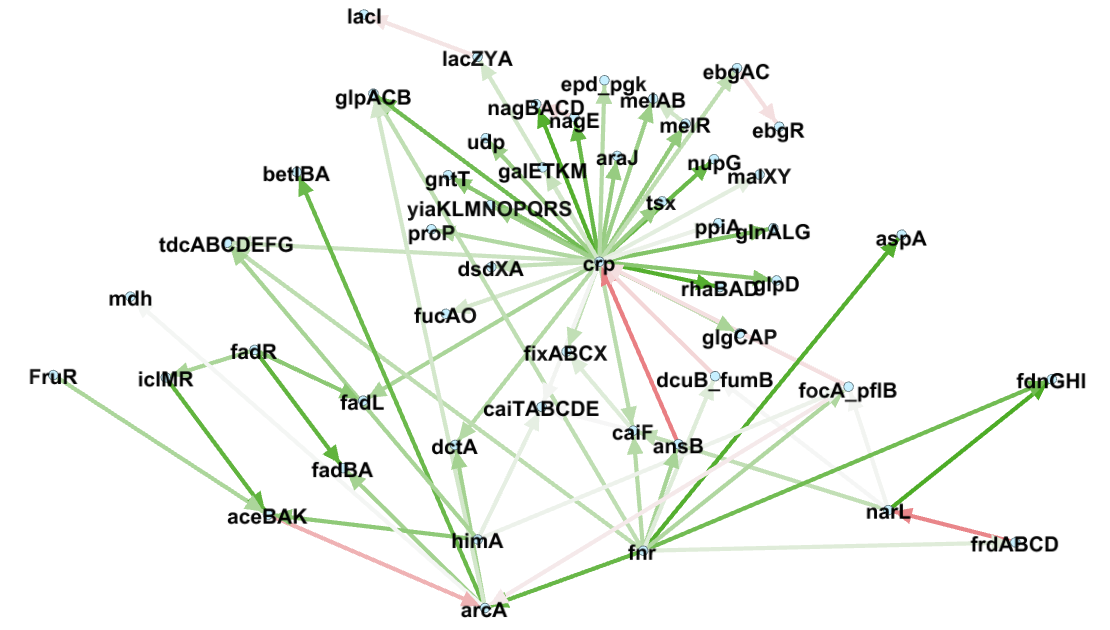}
\caption{Orientation by CGNN of E. coli subnetwork with 50 nodes and corresponding to  Syntren simulation with random seed 0. Green edges are good orientation and red arrows false orientation. The strength of the line refers to the confidence of the algorithm.}
\label{fig:graph_syntren_cgnn}
\end{figure}

\subsubsection{Results on biological real-world data}

CGNN is applied to the protein network problem \cite{sachs2005causal}, using the Anti-CD3/CD28 dataset with  853 observational data points corresponding to general perturbations without specific interventions. All algorithms were given the skeleton of the causal graph \cite[Fig. 2]{sachs2005causal} with same hyper-parameters as in the previous subsection. We run each algorithm on 10-fold cross-validation. Table 6 in Appendix reports average (std. dev.) results.

 Constraint-based algorithms obtain surprisingly low scores, because they cannot identify many V-structures in this graph. We confirm this by evaluating conditional independence tests for the adjacent tuples of nodes \textit{pip3}-\textit{akt}-\textit{pka}, \textit{pka}-\textit{pmek}-\textit{pkc}, \textit{pka}-\textit{raf}-\textit{pkc} and we do not find strong evidences for V-structure. Therefore methods based on distributional asymmetry between cause and effect seem better suited to this dataset. CGNN obtains good results compared to the other algorithms. Notably, Figure \ref{CGNN_cyto} shows that CGNN is able to recover the strong signal transduction pathway \textit{raf}$\rightarrow$\textit{mek}$\rightarrow$\textit{erk} reported in \cite{sachs2005causal} and corresponding to clear direct enzyme-substrate causal effect. CGNN gives important scores for edges with good orientation (green line), and low scores (thinnest edges) to the wrong edges (red line), suggesting that false causal discoveries may be controlled by using the confidence scores defined in Eq. \eqref{score_edges}.
  
\begin{figure}

   \begin{subfigure}{0.5\textwidth}
   \begin{center}
        \begin{tikzpicture}
          \def \h {11}
          \def \radius {1.8cm}
          \def \radiuscircle {0.2cm}
          \fontsize{6.0pt}{6.5pt}\selectfont
          \node[draw, circle,minimum size=\radiuscircle](praf) at ({360/\h * (1 - 1)}:\radius) {$raf$};
          \node[draw, circle,minimum size=\radiuscircle](pmek) at ({360/\h * (2 - 1)}:\radius) {$mek$};
          \node[draw, circle,minimum size=\radiuscircle](plcg) at ({360/\h * (3 - 1)}:\radius) {$plcg$};
          \node[draw, circle,minimum size=\radiuscircle](PIP2) at ({360/\h * (4 - 1)}:\radius) {$pip_2$};
          \node[draw, circle,minimum size=\radiuscircle](PIP3) at ({360/\h * (5 - 1)}:\radius) {$pip_3$};
          \node[draw, circle,minimum size=\radiuscircle](p44/42) at ({360/\h * (6 - 1)}:\radius) {$erk$};
          \node[draw, circle,minimum size=\radiuscircle](pakts473) at ({360/\h * (7 - 1)}:\radius) {$akt$};
          \node[draw, circle,minimum size=\radiuscircle](PKA) at ({360/\h * (8 - 1)}:\radius) {$pka$};
          \node[draw, circle,minimum size=\radiuscircle](PKC) at ({360/\h * (9 - 1)}:\radius) {$pkc$};
          \node[draw, circle,minimum size=\radiuscircle](P38) at ({360/\h * (10 - 1)}:\radius) {$p38$};
          \node[draw, circle,minimum size=\radiuscircle](pjnk) at ({360/\h * (11 - 1)}:\radius) {$jnk$};
          \draw[pil] (PKA) -- (praf);
          \draw[pil] (PKA) -- (P38);
          \draw[pil] (PKC) -- (praf);
          \draw[pil] (PKA) -- (pjnk);
          \draw[pil] (plcg) -- (PKC);
          \draw[pil] (PIP3) -- (plcg);
          \draw[pil] (PKC) -- (P38);
          \draw[pil] (plcg) -- (PIP2);
          \draw[pil] (PKC) -- (pjnk);
          \draw[pil] (PIP2) -- (PKC);
          \draw[pil] (praf) -- (pmek);
          \draw[pil] (PKC) -- (pmek);
          \draw[pil] (PKA) -- (pmek);
          \draw[pil] (PIP2) -- (PIP3);
          \draw[pil] (PIP3) -- (pakts473);
          \draw[pil] (pmek) -- (p44/42);
          \draw[pil] (PKA) -- (pakts473);
          \draw[pil] (PKA) -- (p44/42);
         \end{tikzpicture}
         \caption{Ground truth}
         \end{center}
\end{subfigure}
   \begin{subfigure}{0.5\textwidth}
    \begin{center}
        \begin{tikzpicture}
          \def \h {11}
          \def \radius {1.8cm}
          \def \radiuscircle {0.2cm}
          \fontsize{6.0pt}{6.5pt}\selectfont
          \node[draw, circle,minimum size=\radiuscircle](praf) at ({360/\h * (1 - 1)}:\radius) {$raf$};
          \node[draw, circle,minimum size=\radiuscircle](pmek) at ({360/\h * (2 - 1)}:\radius) {$mek$};
          \node[draw, circle,minimum size=\radiuscircle](plcg) at ({360/\h * (3 - 1)}:\radius) {$plcg$};
          \node[draw, circle,minimum size=\radiuscircle](PIP2) at ({360/\h * (4 - 1)}:\radius) {$pip_2$};
          \node[draw, circle,minimum size=\radiuscircle](PIP3) at ({360/\h * (5 - 1)}:\radius) {$pip_3$};
          \node[draw, circle,minimum size=\radiuscircle](p44/42) at ({360/\h * (6 - 1)}:\radius) {$erk$};
          \node[draw, circle,minimum size=\radiuscircle](pakts473) at ({360/\h * (7 - 1)}:\radius) {$akt$};
          \node[draw, circle,minimum size=\radiuscircle](PKA) at ({360/\h * (8 - 1)}:\radius) {$pka$};
          \node[draw, circle,minimum size=\radiuscircle](PKC) at ({360/\h * (9 - 1)}:\radius) {$pkc$};
          \node[draw, circle,minimum size=\radiuscircle](P38) at ({360/\h * (10 - 1)}:\radius) {$p38$};
          \node[draw, circle,minimum size=\radiuscircle](pjnk) at ({360/\h * (11 - 1)}:\radius) {$jnk$};
          \draw[green, pil] (PKA) -- (praf);
          \draw[green, pil] (PKA) -- (P38);
          \draw[green, pil] (PKC) -- (praf);
          \draw[green, pil] (PKA) -- (pjnk);
          \draw[red, pil] (PKC) -- (plcg);
          \draw[red, pil] (plcg) -- (PIP3);
          \draw[red, pil] (P38) -- (PKC);
          \draw[blue] (PIP2) -- (plcg);
          \draw[red, pil] (pjnk) -- (PKC);
          \draw[red, pil] (PKC) -- (PIP2);
          \draw[green, pil] (praf) -- (pmek);
          \draw[blue] (PKC) -- (pmek);
          \draw[blue] (PKA) -- (pmek);
          \draw[green, pil] (PIP2) -- (PIP3);
          \draw[blue] (PIP3) -- (pakts473);
          \draw[blue] (pmek) -- (p44/42);
          \draw[red, pil] (pakts473) -- (PKA);
          \draw[red, pil] (p44/42) -- (PKA);
         \end{tikzpicture}
         \caption{GES}
         \end{center}
\end{subfigure}
   \begin{subfigure}{0.5\textwidth}
    \begin{center}
        \begin{tikzpicture}
          \def \h {11}
          \def \radius {1.8cm}
          \def \radiuscircle {0.2cm}
          \fontsize{6.0pt}{6.5pt}\selectfont
          \node[draw, circle,minimum size=\radiuscircle](praf) at ({360/\h * (1 - 1)}:\radius) {$raf$};
          \node[draw, circle,minimum size=\radiuscircle](pmek) at ({360/\h * (2 - 1)}:\radius) {$mek$};
          \node[draw, circle,minimum size=\radiuscircle](plcg) at ({360/\h * (3 - 1)}:\radius) {$plcg$};
          \node[draw, circle,minimum size=\radiuscircle](PIP2) at ({360/\h * (4 - 1)}:\radius) {$pip_2$};
          \node[draw, circle,minimum size=\radiuscircle](PIP3) at ({360/\h * (5 - 1)}:\radius) {$pip_3$};
          \node[draw, circle,minimum size=\radiuscircle](p44/42) at ({360/\h * (6 - 1)}:\radius) {$erk$};
          \node[draw, circle,minimum size=\radiuscircle](pakts473) at ({360/\h * (7 - 1)}:\radius) {$akt$};
          \node[draw, circle,minimum size=\radiuscircle](PKA) at ({360/\h * (8 - 1)}:\radius) {$pka$};
          \node[draw, circle,minimum size=\radiuscircle](PKC) at ({360/\h * (9 - 1)}:\radius) {$pkc$};
          \node[draw, circle,minimum size=\radiuscircle](P38) at ({360/\h * (10 - 1)}:\radius) {$p38$};
          \node[draw, circle,minimum size=\radiuscircle](pjnk) at ({360/\h * (11 - 1)}:\radius) {$jnk$};
          \draw[red, pil] (praf) -- (PKA);
          \draw[green, pil] (PKA) -- (P38);
          \draw[red, pil] (praf) -- (PKC);
          \draw[red, pil] (pjnk) -- (PKA);
          \draw[green, pil] (plcg) -- (PKC);
          \draw[green, pil] (PIP3) -- (plcg);
          \draw[green, pil] (PKC) -- (P38);
          \draw[green] (plcg) -- (PIP2);
          \draw[red, pil] (pjnk) -- (PKC);
          \draw[red, pil] (PKC) -- (PIP2);
          \draw[green, pil] (praf) -- (pmek);
          \draw[red, pil] (pmek) -- (PKC);
          \draw[red, pil] (pmek) -- (PKA);
          \draw[green, pil] (PIP2) -- (PIP3);
          \draw[red, pil] (pakts473) -- (PIP3);
          \draw[green, pil] (pmek) -- (p44/42);
          \draw[green, pil] (PKA) -- (pakts473);
          \draw[green, pil] (PKA) -- (p44/42);
         \end{tikzpicture}
         \caption{CAM}
         \end{center}
\end{subfigure}
    \begin{subfigure}{0.5\textwidth}
    \begin{center}
        \begin{tikzpicture}
          \def \h {11}
          \def \radius {1.8cm}
          \def \radiuscircle {0.2cm}
          \fontsize{6.0pt}{6.5pt}\selectfont
          \node[draw, circle,minimum size=\radiuscircle](praf) at ({360/\h * (1 - 1)}:\radius) {$raf$};
          \node[draw, circle,minimum size=\radiuscircle](pmek) at ({360/\h * (2 - 1)}:\radius) {$mek$};
          \node[draw, circle,minimum size=\radiuscircle](plcg) at ({360/\h * (3 - 1)}:\radius) {$plcg$};
          \node[draw, circle,minimum size=\radiuscircle](PIP2) at ({360/\h * (4 - 1)}:\radius) {$pip_2$};
          \node[draw, circle,minimum size=\radiuscircle](PIP3) at ({360/\h * (5 - 1)}:\radius) {$pip_3$};
          \node[draw, circle,minimum size=\radiuscircle](p44/42) at ({360/\h * (6 - 1)}:\radius) {$erk$};
          \node[draw, circle,minimum size=\radiuscircle](pakts473) at ({360/\h * (7 - 1)}:\radius) {$akt$};
          \node[draw, circle,minimum size=\radiuscircle](PKA) at ({360/\h * (8 - 1)}:\radius) {$pka$};
          \node[draw, circle,minimum size=\radiuscircle](PKC) at ({360/\h * (9 - 1)}:\radius) {$pkc$};
          \node[draw, circle,minimum size=\radiuscircle](P38) at ({360/\h * (10 - 1)}:\radius) {$p38$};
          \node[draw, circle,minimum size=\radiuscircle](pjnk) at ({360/\h * (11 - 1)}:\radius) {$jnk$};
          \draw[red, pil, line width = 0.30pt] (praf) -- (PKA);
          \draw[green, pil, line width = 0.32pt] (PKA) -- (P38);
          \draw[green, pil, line width = 0.33pt] (PKC) -- (praf);
          \draw[red, pil, line width = 0.35pt] (pjnk) -- (PKA);
          \draw[red, pil, line width = 0.36pt] (PKC) -- (plcg);
          \draw[green, pil, line width = 0.36pt] (PIP3) -- (plcg);
          \draw[green, pil, line width = 0.38pt] (PKC) -- (P38);
          \draw[red,pil, line width = 0.40pt] (PIP2) -- (plcg);
          \draw[green, pil, line width = 0.43pt] (PKC) -- (pjnk);
          \draw[green, pil, line width = 0.44pt] (PIP2) -- (PKC);
          \draw[green, pil, line width = 0.48pt] (praf) -- (pmek);
          \draw[green, pil, line width = 0.52pt] (PKC) -- (pmek);
          \draw[green, pil, line width = 0.62pt] (PKA) -- (pmek);
          \draw[green, pil, line width = 0.89pt] (PIP2) -- (PIP3);
          \draw[green, pil, line width = 1.38pt] (PIP3) -- (pakts473);
          \draw[green, pil, line width = 1.47pt] (pmek) -- (p44/42);
          \draw[green, pil, line width = 1.55pt] (PKA) -- (pakts473);
          \draw[green, pil, line width = 1.60pt] (PKA) -- (p44/42);
         \end{tikzpicture}
         \caption{CGNN}
         \end{center}
\end{subfigure}

  \caption{Causal protein network } 
  %Solid/dash arrows correspond to correct/incorrect edge orientations. The width of an arrow refers to the confidence of CGNN in the prediction~\eqref{eq:conf}.}
\label{CGNN_cyto}
\end{figure}
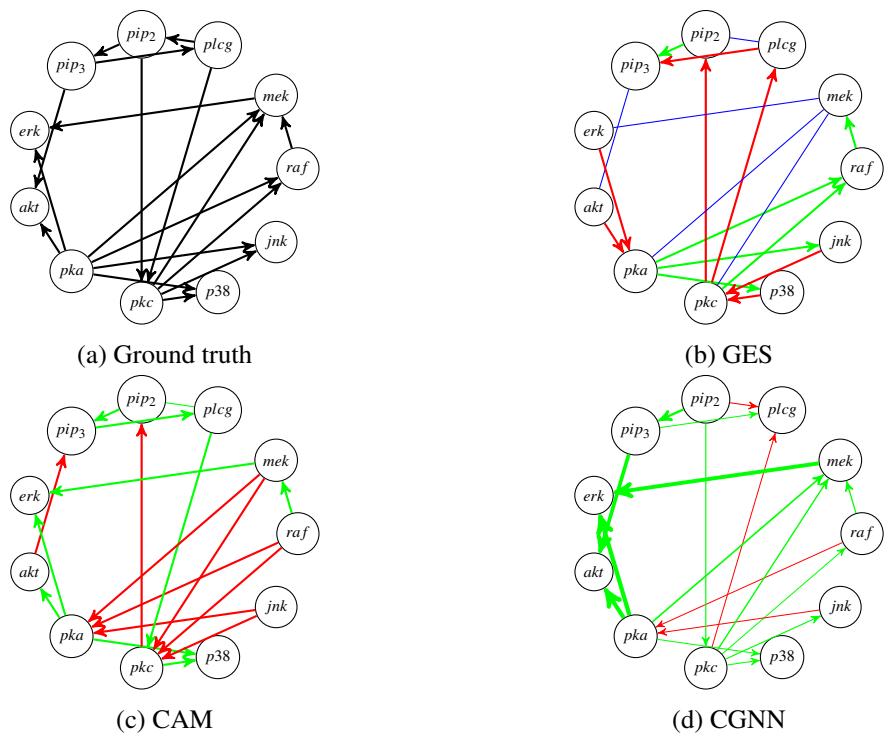

\section{Towards predicting confounding effects
\label{confounder}}

In this subsection we propose an extension of our algorithm relaxing the causal sufficiency assumption. We are still assuming the Causal Markov and faithfulness assumptions, thus three options  have to be considered for each edge $(X_i,X_j)$ of the skeleton representing a direct dependency: $X_i \rightarrow X_j$, $X_j \rightarrow X_i$ and $X_i\leftrightarrow X_j$ (both variables are consequences of common hidden variables).

\subsection{Principle}

Hidden common causes are modeled through correlated random noise. Formally, an additional noise variable $E_{i,j}$ is associated to each $X_i-X_j$ edge in the graph skeleton.

We use such new models with correlated noise to study the robustness of our graph reconstruction algorithm to increasing violations of causal sufficiency, by occluding variables from our datasets. For example, consider the FCM on $\mathbf{X} = [X_1, \ldots, X_5]$ that was presented on Figure \ref{figure:causalnetwork}. If variable $X_1$ would be missing from data, the correlated noise $E_{2,3}$ would be responsible for the existence of a double headed arrow connection  $X_2\leftrightarrow X_3$ in the skeleton of our new type of model. The resulting FCM is shown in Figure  \ref{figure:FCM_confounders}. Notice that direct causal effects such as $X_3 \rightarrow X_5$ or $X_4 \rightarrow X_5$ may persist, even in presence of possible confounding effects.

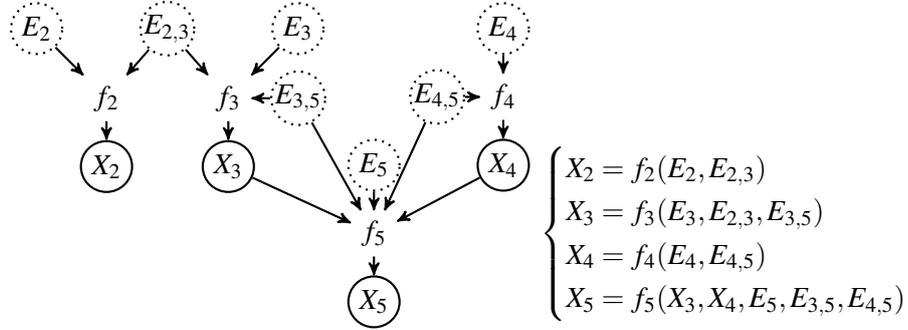
\begin{figure}[h]
    \centering
    \begin{tikzpicture}[node distance=0.25cm, auto,]

        \node[punkt, dotted] (e23) at (0,0) {$E_{2,3}$};
        \node[punkt, dotted, right=1cm of e23] (e3) {$E_3$};
        \node[punkt, dotted, left=1cm of e23] (e2) {$E_2$};
        \node[punkt, dotted, right=2cm of e3] (e4) {$E_4$};
        \node[below=of e4] (f4) {$f_4$};
        \node[punkt, below=of f4] (x4) {$X_4$};
        \node[punkt, dotted, left=1cm of x4] (e5) {$E_5$};
        \node[punkt, dotted, left=0.2cm of f4] (e45) {$E_{4,5}$};
        \node[below=of e23, shift={(-0.8cm,0)}] (f2) {$f_2$};
        \node[right=1cm of f2] (f3) {$f_3$};
        \node[punkt, dotted, right=0.25cm of f3] (e35) {$E_{3,5}$};
        \node[punkt, below=of f3] (x3) {$X_3$};
        \node[below=of e5] (f5) {$f_5$};
        \node[punkt, below=of f5] (x5) {$X_5$};
        \node[punkt, below=of f2] (x2) {$X_2$};
        
        \draw[pil] (e2) -- (f2);
        \draw[pil] (e3) -- (f3);
        \draw[pil] (e4) -- (f4);
        \draw[pil] (e5) -- (f5);
        
        \draw[pil] (e23) -- (f2);
        \draw[pil] (e23) -- (f3);
        \draw[pil] (e35) -- (f3);
        \draw[pil] (e45) -- (f4);
        \draw[pil] (e35) -- (f5);
        \draw[pil] (e45) -- (f5);
        
        \draw[pil] (f2) -- (x2);
        \draw[pil] (f3) -- (x3);
        \draw[pil] (f4) -- (x4);
        \draw[pil] (f5) -- (x5);
        
        \draw[pil] (x3) -- (f5);
        \draw[pil] (x4) -- (f5);
        
        \node[right=1.75cm of f5] {
                $\begin{cases}   
                X_2 = f_2(E_2, E_{2,3}) \\
                X_3 = f_3(E_3, E_{2,3},E_{3,5}) \\
                X_4 = f_4(E_4, E_{4,5}) \\
                X_5 = f_5(X_3,X_4,E_5,E_{3,5}, E_{4,5}) \\
                \end{cases}$};
    \end{tikzpicture}
    \caption{The Functional Causal Model (FCM) on $\mathbf{X} = [X_1, \ldots, X_5]$ with the missing variable $X_1$}
    \label{figure:FCM_confounders}
\end{figure}

Formally, given a graph skeleton $\mathcal{S}$, the FCM with correlated noise variables is defined as:

\begin{equation}
    \label{eq:FCM_confounders}
    X_i \leftarrow f_i(X_{\Pa{i}}, E_i, E_{\Ne{i}}),
\end{equation}

where $\Ne{i}$ is the set of indices of all the variables adjacent to variable $X_i$ in the skeleton $\mathcal{S}$. 

One can notice that this model corresponds to the most general formulation of the FCM with potential confounders for each pair of variables in a given skeleton (representing direct dependencies) where each random variable $E_{i,j}$ summarizes all the unknown influences of (possibly multiple) hidden variables  influencing the two variables $X_i$ and $X_j$.

Here we make a clear distinction between the directed acyclic graph denoted $\mathcal{G}$ and the skeleton $\mathcal{S}$. Indeed, due to the presence of confounding correlated noise, any variable in $\mathcal{G}$ can be removed without altering $\mathcal{S}$. We use the same generative neural network to model the new FCM presented in Equation \ref{eq:FCM_confounders}. The difference is the new noise variables having effect on pairs of variables  simultaneously. However, since the correlated noise FCM is still defined over a directed acyclic graph $\mathcal{G}$, the functions $\hat{f}_1, \ldots, \hat{f}_d$ of the model, which we implement as neural networks, the model can still be learned end-to-end using backpropagation based on the CGNN loss.   

All edges are evaluated with these correlated noises, the goal being to see whether introducing a correlated noise explains the dependence between the two variables $X_i$ and $X_j$. 

As mentioned before, the score used by CGNN is:
\begin{equation}
  S(\mathcal{C}_{\hat{\cal G},\hat{f}}, \mathcal{D}) =
   \widehat{\text{MMD}}_k(\mathcal{D}, \widehat{\mathcal{D}}) + \lambda 
    |\widehat{\mathcal{G}}|
  \label{loss_cgnn}
\end{equation}

where $|\widehat{\mathcal{G}}|$ is the total number of edges in the DAG. 
In the graph search, for any given edge, we compare the score associated to the graph considered with and without this edge. If the contribution of this edge is negligible compared to a given threshold lambda, the edge is considered as spurious.

The non-parametric optimization of the $\hat{\mathcal{G}}$ structure is also achieved using a Hill-Climbing algorithm; in each step an edge of $\mathcal{S}$ is randomly drawn and modified in $\hat{\mathcal{G}}$ using one out of the possible three operators: reverse the edge, add an edge and remove an edge. Other algorithmic details are as in Section \ref{sec:nonparam}: the greedy search optimizes the penalized loss function (Eq. \ref{loss_cgnn}).
For CGNN, we set the hyperparameter $\lambda = 5 \times 10^{-5}$ fitted on the training graph dataset. 

The algorithm stops when no improvement is obtained. Each causal  edge $X_i \rightarrow X_j$ in $\mathcal{G}$ is associated with a score, measuring its contribution to the global score:

\begin{equation}
S_{X_i \rightarrow X_j} = S(\mathcal{C}_{\hat{\mathcal{G}} - \{X_i \rightarrow X_j\},\hat{f}}, \mathcal{D}) - S(\mathcal{C}_{\hat{\cal G},\hat{f}}, \mathcal{D}) 
\end{equation}

Missing edges are associated with a score 0.

\subsection{Experimental validation}
\paragraph{Benchmarks.} The empirical validation of this extension of CGNN is conducted on same benchmarks as in Section \ref{sec:multi} ($\mathcal{G}_i$, $i \in [[2,5]]$), where 3 variables (causes for at least two other variables in the graph) have been randomly removed.\footnote{The datasets considered are available at \url{http://dx.doi.org/10.7910/DVN/UZMB69}} The true graph skeleton is augmented with edges $X-Y$ for all $X,~Y$ that are consequences of a same removed cause. All algorithms are provided with the same graph skeleton for a fair comparison.  The task is to both orient the edges in the skeleton, and remove the spurious direct dependencies created by latent causal variables.

\paragraph{Baselines.} CGNN is compared with state of art methods: i) constraint-based RFCI \citep{colombo2012learning}, extending the PC method equipped with Gaussian conditional independence test (RFCI-Gaussian) and the gamma HSIC conditional independence test \citep{gretton2005kernel} (RFCI-HSIC).  We use the order-independent constraint-based version proposed by \cite{colombo2014order} and the majority rules for the orientation of the edges. For CGNN, we set the hyperparameter $\lambda = 5 \times 10^{-5}$ fitted on the training graph dataset.  Jarfo is trained on the 16,200 pairs of the cause-effect pair challenge \citep{guyon2013cepc,IGuyon2014} to detect for each pair of variable if $X_i \rightarrow Y_i$, $Y_i \rightarrow X_i$ or $X_i \leftrightarrow Y_i$. 

\begin{table}[h!]
\footnotesize
  \caption{AUPR, SHD and SID on causal discovery with confounders. $^*$ denotes significance at $p=10^{-2}$.}
  \centering
  \begin{tabular}{lcccr}
    \toprule
    method & AUPR &  SHD & SID \\
    \midrule
    RFCI-Gaussian & 0.22 (0.08) & 21.9 (7.5) & 174.9 (58.2) \\
    RFCI-HSIC & 0.41 (0.09) & 17.1 (6.2) & 124.6 (52.3) \\
    Jarfo & 0.54 (0.21) & 20.1 (14.8) & 98.2 (49.6) \\
    \midrule
    \textbf{CGNN} ($\widehat{\text{MMD}}_k$) & \underline{0.71}* (0.13) & \underline{11.7}* (5.5) & \underline{53.55}* (48.1)\\
    \bottomrule
  \end{tabular}
  \label{table:acc_graph_confounder}
\end{table} 

\paragraph{Results.} comparative performances are shown in Table \ref{table:acc_graph_confounder}, reporting the area under the precision/recall curve. Overall, these results confirm the robustness of the CGNN proposed approach w.r.t. confounders, and its competitiveness w.r.t. RFCI with powerful conditional independence test (RFCI-HSIC). Interestingly, the effective causal relations between the visible variables are associated with a high score; spurious links due to hidden latent variables get a low score or are removed.

\section{Discussion and Perspectives \label{discussion}}

This paper introduces CGNN, a new framework and methodology for functional causal model learning, leveraging the power and non-parametric flexibility of Generative Neural Networks. 

CGNN seamlessly accommodates causal modeling in presence of confounders, and its extensive empirical validation demonstrates its merits compared to the state of the art on medium-size problems. We believe that our approach opens new avenues of research, both from the point of view of leveraging the power of deep learning in causal discovery and from the point of view of building deep networks with better structure \textit{interpretability}. Once the model is learned, the CGNNs present the advantage to be fully parametrized and may be used to simulate interventions on one or more variables of the model and evaluate their impact on a set of target variables. This usage is relevant in a wide variety of domains, typically among medical and sociological domains.

The main limitation of CGNN is its computational cost, due to the quadratic complexity of the CGNN learning criterion w.r.t. the data size, based on the Maximum Mean Discrepancy between the generated and the observed data. A linear approximation thereof has been proposed, with comparable empirical performances.

The main perspective for further research aims at a better scalability of the approach from medium to large problems. On the one hand, the computational scalability could be tackled by using embedded framework for the structure optimization (inspired by lasso methods). Another perspective regards the extension of the approach to categorical variables.

\bibliographystyle{apalike}
\bibliography{biblio.bib}   % name your BibTeX data base

\clearpage
\newpage

\section{Appendix}

\subsection{The Maximum Mean Discrepancy (MMD) statistic}
\label{app:mmd}

The Maximum Mean Discrepancy (MMD) statistic \citep{gretton2007kernel}
measures the distance between two probability distributions $P$ and $\hat{P}$,
defined over $\mathbb{R}^d$, as the real-valued quantity
\begin{equation*}
    \text{MMD}_k(P, \hat{P}) = \left\| \mu_k(P) - \mu_k(\hat{P})
    \right\|_{\mathcal{H}_k}.
\end{equation*}
Here, $\mu_k = \int k(x, \cdot) \mathrm{d} P(x)$ is the \emph{kernel mean
embedding} of the distribution $P$, according to the real-valued symmetric
kernel function $k(x, x') = \langle k(x, \cdot), k(x', \cdot)
\rangle_{\mathcal{H}_k}$ with associated reproducing kernel Hilbert space
$\mathcal{H}_k$.  Therefore, $\mu_k$ summarizes $P$ as the expected value of
the features computed by $k$ over samples drawn from $P$. 

In practical applications, we do not have access to the distributions $P$ and
$\hat{P}$, but to their respective sets of samples $\mathcal{D}$ and
$\hat{\mathcal{D}}$, defined in Section~\ref{ScoringMetric}. In this case, we
approximate the kernel mean embedding $\mu_k(P)$ by the \emph{empirical kernel
mean embedding} $\mu_k(\mathcal{D}) = \frac{1}{|\mathcal{D}|} \sum_{x \in
\mathcal{D}} k(x, \cdot)$, and respectively for $\hat{P}$. Then, the empirical
MMD statistic is 
\begin{align*}
\widehat{\text{MMD}}_k(\mathcal{D}, \hat{\mathcal{D}}) &= \left\|
\mu_k(\mathcal{D}) - \mu_k(\hat{\mathcal{D}}) \right\|_{\mathcal{H}_k}
\hspace{-0.25cm}
=\frac{1}{n^2} \sum_{i, j}^{n} k(x_i, x_j) +
\frac{1}{n^2} \sum_{i, j}^{n} k(\hat{x}_i, \hat{x}_j)
- \frac{2}{n^2} \sum_{i,j}^n k(x_i, \hat{x}_j).
\end{align*}

Importantly, the empirical MMD tends to zero as $n \to \infty$ if and only if
$P = \hat{P}$, as long as $k$ is a characteristic kernel
\citep{gretton2007kernel}. This property makes the MMD an excellent choice to
model how close the observational distribution $P$ is to the estimated
observational distribution $\hat{P}$. Throughout this paper, we will employ
a particular characteristic kernel: the Gaussian kernel $k(x, x') = \exp(-
\gamma \| x - x' \|_2^2)$, where $\gamma > 0$ is a hyperparameter controlling
the smoothness of the features. 

In terms of computation, the evaluation of $\text{MMD}_k(\mathcal{D},
\hat{\mathcal{D}})$ takes $O(n^2)$ time, which is prohibitive for large $n$.
When using a shift-invariant kernel, such as the Gaussian kernel, one can
invoke Bochner's theorem \citep{edwards1964fourier} to obtain a linear-time
approximation to the empirical MMD \citep{lopez2015towards}, with form 
\begin{equation*}
\widehat{\text{MMD}}^m_k(\mathcal{D}, \hat{\mathcal{D}}) = 
\left\|
  \hat{\mu}_k(\mathcal{D}) - \hat{\mu}_k(\hat{\mathcal{D}})
\right\|_{\mathbb{R}^m}
\label{eq:approx_mmd}
\end{equation*}
and $O(mn)$ evaluation time. Here, the \emph{approximate empirical kernel mean
embedding} has form
\begin{equation*}
  \hat{\mu}_k(\mathcal{D}) = \sqrt{\frac{2}{m}} \frac{1}{|\mathcal{D}|} \sum_{x
  \in \mathcal{D}} \left[ \cos(\langle w_1, x \rangle + b_1), \ldots,
  \cos(\langle w_m, x \rangle + b_m) \right],
\end{equation*}
where $w_i$ is drawn from the normalized Fourier transform of $k$, and $b_i
\sim U[0, 2\pi]$, for $i=1, \ldots, m$.  In our experiments, we compare the
performance and computation times of both $\widehat{\text{MMD}}_k$ and
$\widehat{\text{MMD}}_k^m$. 

\subsection{Proofs \label{app:proofs}}

\setcounter{prop}{0}

\begin{prop}
 Let $\textbf{X} = [X_1, \ldots, X_d]$ denote a set of continuous random variables with joint distribution $P$, and further assume that the joint density function $h$ of $P$ is continuous and strictly positive on a compact and convex subset of $\mathbb{R}^{d}$, and zero elsewhere. Letting $\cal G$ be a DAG such that $P$ can be factorized along $\cal G$, 
 $$ P(X) = \prod_i P(X_i | X_{\Pa{i}})$$
 there exists $f = (f_1, \ldots, f_d)$ with $f_i$ a continuous function with compact support in $\mathbb{R}^{|\Pa{i}|}\times [0,1]$ such that $P(X)$ equals the generative model defined from FCM $({\cal G}, f, {\cal E})$, with ${\cal E} = \mathcal{U}[0,1]$ the uniform distribution on $[0,1]$.
\end{prop}

\begin{proof}
By induction on the topological order of $\cal G$. Let $X_i$ be such that $|\Pa{i}|=0$ and consider the cumulative distribution $F_i(x_i)$ defined over the domain of $X_i$ ($F_i(x_i) = Pr(X_i < x_i)$). $F_i$ is strictly monotonous as the joint density function is strictly positive therefore its inverse, the quantile function $Q_i: [0,1] \mapsto dom(X_i)$ is defined and continuous. By construction, $Q_i(e_i) =F_i^{-1}(e_i)$ and setting $Q_i = f_i$ yields the result.\\
Assume $f_i$ be defined for all variables $X_i$ with topological order less than $m$.  Let $X_j$ with topological order $m$ and $Z$ the vector of its parent variables. For any noise vector $e = (e_i, i \in \Pa{j})$ let $z = (x_i,  i \in \Pa{j})$ be the value vector of variables in $Z$ defined from $e$. The conditional cumulative distribution $F_j(x_j | Z=z) = Pr(X_j < x_j | Z=z)$ is strictly continuous and monotonous wrt $x_j$, and can be inverted using the same argument as above. Then we can define $f_j(z,e_j) = F_j^{-1}(z,e_j)$.

Let $K_j = dom(X_j)$ and $K_{\Pa{j}} = dom(Z)$. We will show now that the function $f_j$ is continuous on $K_{\Pa{j}} \times [0,1]$, a compact subset of $\mathbb{R}^{|\Pa{j}|}\times [0,1]$.

By assumption, there exist $a_j \in \mathcal{R}$ such that, for $(x_j, z) \in K_j \times K_{\Pa{j}}$, $F(x_j|z) = \int_{a_j}^{x_j} \frac{h_j(u,z)}{h_j(z)} \mathrm{d}u$, with $h_j$ a continuous and strictly positive density function. For $(a,b) \in K_j \times K_{\Pa{j}}$, as the function $(u, z) \rightarrow \frac{h_j(u,z)}{h_j(z)}$ is continuous on the compact $K_j \times K_{\Pa{j}}$, $\lim\limits_{\substack{x_j \rightarrow a}} F(x_j|z) = \int_{a_j}^{a} \frac{h_j(u,z)}{h_j(z)} \mathrm{d}u$ uniformly on $K_{\Pa{j}}$ and $\lim\limits_{\substack{z \rightarrow b}} F(x_j|z) = \int_{a_j}^{x_j} \frac{h_j(u,b)}{h_j(b)}$ on $K_j$, according to exchanging limits theorem, $F$ is continuous on $(a,b)$. 
 
For any sequence $z_n \rightarrow z$, we have that $F(x_j|z_n) \rightarrow F(x_j|z)$ uniformly in $x_j$. Let define two sequences $u_n$ and $x_{j,n}$, respectively on $[0,1]$ and $K_j$, such that $u_n \rightarrow u$ and $x_{j,n} \rightarrow x_{j}$. As $F(x_j|z)=u$ has unique root $x_j = f_j(z,u)$, the root of $F(x_j|z_n)=u_n$, that is, $x_{j,n} = f_j(z_n,u_n)$ converge to $x_j$. Then the function $(z,u) \rightarrow f_j(z,u)$ is continuous on $K_{\Pa{i}} \times [0,1]$.
\end{proof}

\begin{prop}
For $m \in [[1,d]]$, let $Z_m$ denote the set of variables with topological order less than $m$ and let $d_m$ be its size. For any $d_m$-dimensional vector of noise values $e^{(m)}$, let $z_m(e^{(m)})$ (resp. $\widehat{z_m}(e^{(m)})$) be the vector of values computed in topological order from the FCM $({\cal G}, f, {\cal E})$ (resp. the CGNN $({\cal G}, \hat{f}, {\cal E})$). 
For any $\epsilon > 0$, there exists a set of networks $\hat{f}$ with architecture $\cal G$ such that 
\begin{equation}
\forall e^{(m)},  \|z_m(e^{(m)})- \widehat{z_m}(e^{(m)})\| < \epsilon
\label{eq:prop2}
\end{equation}
\end{prop}

\begin{proof}
By induction on the topological order of $\cal G$. Let $X_i$ be such that $|\Pa{i}|=0$. 
Following the universal approximation theorem \cite{cybenko1989approximation}, as $f_i$ is a continuous function over a compact of $\mathbb{R}$, there exists a neural net $\hat{f_{i}}$ such that $\|f_i - \hat{f_{i}}\|_\infty < \epsilon/d_1$. Thus  Eq. \ref{eq:prop2} holds for the set of networks $\hat{f_i}$ for $i$ ranging over variables with topological order 0.\\
Let us assume that Prop. 2 holds up to $m$, and let us assume for brevity that there exists a single variable $X_j$ with topological order $m +1$. Letting $\hat{f_j}$ be such that $\|f_j - \hat{f_j}\|_\infty < \epsilon/3$ (based on the universal approximation property), letting $\delta$ be such that for all $u$ $\|\hat f_j(u) - \hat f_j(u+\delta)\|< \epsilon/3$ (by absolute continuity) and letting $\hat f_i$ satisfying Eq. \ref{eq:prop2} for $i$ with topological order less than $m$ for $min(\epsilon/3,\delta)/d_{m}$, it comes: 
$\|(z_m,f_j(z_m,e_j)) - (\hat z_m,\hat{f_j}(\hat{z_m}, e_j))\| \le \|z_m - \hat z_m\| + |f_j(z_m,e_j) - \hat{f_j}(z_m, e_j)| +  | \hat{f_j}(z_m,e_j) - \hat{f_j}(\hat{z_m}, e_j)| < \epsilon/3 + \epsilon/3 + \epsilon/3$, which ends the proof.
\end{proof}

\begin{prop}
Let  $\mathcal{D}$ be an infinite observational sample generated from $({\cal G}, f, {\cal E})$.
With same notations as in Prop. 2, for every  sequence $\epsilon_t$ such that $\epsilon_t >0$ goes to zero when $t \rightarrow \infty$, there exists a set $\widehat{f_t} = (\hat f^{t}_1 \ldots \hat f^{t}_d)$ such that $\widehat{\text{MMD}_k}$ between $\cal D$ and an infinite size sample $\widehat{\cal D}_{t}$ generated from the CGNN $({\cal G},\widehat{f_{t}},\cal E)$ is less than $\epsilon_t$.
 \end{prop}
 
\begin{proof}
According to Prop. \ref{prop2} and with same notations, letting $\epsilon_t > 0$ go to 0 as $t$ goes to infinity, consider  ${\hat f}_t=(\hat f^{t}_1 \ldots \hat f^{t}_d)$ and $\hat{z_t}$ defined from ${\hat f}_t$ such that for all $e \in [0,1]^d$, $\|z(e)- \widehat{z}_t(e)\| < \epsilon_t$. 

Let $\{ \hat{\mathcal{D}_t} \}$ denote the infinite sample generated after $\hat{f_t}$.
The score of the CGNN $(\mathcal{G},\hat{f_t},{\cal E})$ is $ \widehat{\text{MMD}}_k(\mathcal{D}, \hat{\mathcal{D}_t}) =  \mathbb{E}_{e,e'}[k(z(e),z(e')) - 2  k(z(e),\widehat{z}_t(e')) + k(\widehat{z}_t(e), \widehat{z}_t(e'))]$.

As $\hat{f_t}$ converges towards $f$ on the compact $[0,1]^d$, using the bounded convergence theorem on a compact subset of $\mathbb{R}^{d}$,  $\widehat{z_t}(e) \rightarrow z(e)$ uniformly for $t \rightarrow \infty$, it follows from the Gaussian kernel function being bounded and continuous that $\widehat{\text{MMD}}_k(\mathcal{D}, \hat{\mathcal{D}_t}) \rightarrow 0$, when $t \rightarrow \infty$.
\end{proof}

\begin{prop}
Let $\textbf{X} = [X_1, \ldots, X_d]$ denote a set of continuous random variables with joint distribution $P$, generated by a CGNN $\mathcal{C}_{\mathcal{G},f} = (\mathcal{G}, f, \mathcal{E})$ with ${\cal G}$, a directed acyclic graph. And let $\mathcal{D}$ be an infinite observational sample generated from this CGNN. We assume that $P$ is Markov and faithful to the graph ${\cal G}$, and that every  pair of variables $(X_i,X_j)$ that are  d-connected in the graph are not independent. We note $\widehat{\mathcal{D}}$ an infinite sample generated by a candidate CGNN, $\mathcal{C}_{\widehat{\mathcal{G}},\hat{f}} = (\widehat{\mathcal{G}}, \hat{f}, \mathcal{E})$. Then, \\
(i)\ \ If $\widehat{\mathcal{G}} = {\cal G}$ and $\hat{f} = f$, then $\widehat{\text{MMD}}_k(\mathcal{D}, \widehat{\mathcal{D}}) = 0$.\\
(ii) For any  graph $\widehat{\mathcal{G}}$ characterized by the same adjacencies but not belonging to the Markov equivalence class of $\mathcal{G}$, for all $\hat{f}$, $\widehat{\text{MMD}}_k(\mathcal{D}, \widehat{\mathcal{D}}) \neq 0$.
\end{prop}

\begin{proof} 
The proof of (i) is obvious, as with $\widehat{\mathcal{G}} = {\cal G}$ and $\hat{f} = f$, the joint distribution $\hat{P}$ generated by $\mathcal{C}_{\widehat{\mathcal{G}},\hat{f}} = (\widehat{\mathcal{G}}, \hat{f}, \mathcal{E})$ is equal to $P$, thus we have $\widehat{\text{MMD}}_k(\mathcal{D}, \widehat{\mathcal{D}}) = 0$. 

(ii) Let consider $\widehat{\mathcal{G}}$ a DAG characterized by the same adjacencies but that do not belong to the Markov equivalence class of ${\cal G}$. 
According to \cite{Verma:1990:ESC:647233.719736}, as the DAG ${\cal G}$ and $\widehat{\mathcal{G}}$ have the same adjacencies but are not Markov equivalent, there are not characterized by the same v-structures. 

a) First, we consider that a v-structure $\{X, Y, Z\}$ exists in  ${\cal G}$, but not in $\widehat{\mathcal{G}}$. 
As the distribution P is faithful to ${\cal G}$ and  $X$ and $Z$ are not d-separated by $Y$ in  ${\cal G}$, we have that $(X \nindep Z|Y)$ in $P$. 
Now we consider the graph $\widehat{\mathcal{G}}$. Let $\hat{f}$ be a set of neural networks. We note $\hat{P}$ the distribution generated by the CGNN $\mathcal{C}_{\widehat{\mathcal{G}},\hat{f}}$. As $\widehat{\mathcal{G}}$ is a directed acyclic graph and the variables $E_i$ are mutually independent, $\hat{P}$ is Markov with respect to $\widehat{\mathcal{G}}$. As $\{X, Y, Z\}$ is not a v-structure in $\widehat{\mathcal{G}}$, $X$ and $Z$ are d-separated by $Y$. By using the causal Markov assumption, we obtain that $(X \indep Z|Y)$ in $\hat{P}$.

b) Second, we consider that a v-structure $\{X, Y, Z\}$ exists in  $\widehat{\mathcal{G}}$, but not in ${\cal G}$. 
As $\{X, Y, Z\}$ is not a v-structure in ${\cal G}$, there is an "unblocked path" between the variables $X$ and $Z$, the variables $X$ and $Z$ are d-connected. By assumption, there do not exist a set $D$ not containing $Y$ such that $(X \indep Z|D)$ in $P$. In $\widehat{\mathcal{G}}$, as $\{X, Y, Z\}$ is a v-structure, there exists a set $D$ not containing $Y$ that d-separates $X$ and $Z$. As for all CGNN $\mathcal{C}_{\widehat{\mathcal{G}},\hat{f}}$ generating a distribution $\hat{P}$, $\hat{P}$ is Markov with respect to $\widehat{\mathcal{G}}$, we have that $X \indep Z|D$ in $\hat{P}$.

In the two cases a) and b) considered above, $P$ and $\hat{P}$ do not encode the same conditional independence relations, thus are not equal. We have then  $\widehat{\text{MMD}}_k(\mathcal{D}, \mathcal{D}') \neq 0$.
\end{proof}

\newpage

\subsection{Table of scores for the experiments on cause-effect pairs}

\begin{table}[h!]
  \footnotesize
  \caption{Cause-effect relations: Area Under the Precision Recall curve on 5 benchmarks for the cause-effect experiments (weighted accuracy in parenthesis for T\"ub). Underline values correspond to best scores.}
  \label{table:pairwise}
  \centering
  \begin{tabular}{lccccc} 
    \toprule
    method & {Cha}  & {Net} & {Gauss} & {Multi} & {T\"ub }\\
    \midrule
    %\textit{Regression based} \\
    Best fit & 56.4 & 77.6 & 36.3 & 55.4 & 58.4 (44.9) \\
    LiNGAM & 54.3 & 43.7 & 66.5 & 59.3 & 39.7 (44.3) \\
    %\midrule
    %\textit{Cond. distribution} \\
    CDS  & 55.4 & 89.5 & 84.3 & 37.2 & 59.8 (65.5) \\
    %\midrule
    %\textit{Entropy evaluation} \\
    IGCI  & 54.4 & 54.7 & 33.2 & 80.7 & 60.7 (62.6) \\
    %\midrule
    %\textit{Independence tests} \\
    ANM  & 66.3 & 85.1 & 88.9 & 35.5 & 53.7 (59.5) \\
    PNL  & 73.1 & 75.5 & 83.0 & 49.0 & 68.1 (66.2) \\
    %\midrule
    %\textit{Data-driven} \\
    Jarfo  & \underline{79.5} & \underline{92.7} & 85.3 & 94.6 & 54.5 (59.5)  \\
    %\midrule
    %\textit{Generative} \\
    GPI  & 67.4 & 88.4 & \underline{89.1} & 65.8 & 66.4 (62.6) \\
    \midrule
    \textbf{CGNN} ($\widehat{\text{MMD}}_k$) & 73.6 & 89.6 & 82.9 & \underline{96.6} & \underline{79.8} (74.4)  \\
    \textbf{CGNN} ($\widehat{\text{MMD}}^m_k$) & 76.5 & 87.0 & 88.3 & 94.2 & 76.9 (72.7) \\
    \bottomrule
  \end{tabular}
  \label{TableResultsPairwise}
\end{table}

\subsection{Table of scores for the experiments on graphs}

\begin{table*}[h!]
  \caption{Average (std. dev.) results for the orientation of 20 artificial graphs given true skeleton (left), artificial graphs given skeleton with 20\% error (middle). $^*$ denotes statistical significance at $p=10^{-2}$. Underline values correspond to best scores.}
  \footnotesize
  \hspace*{-1.5cm}
    \centering
  \begin{tabular}{l|ccc|ccc|}
    \toprule
     &  \multicolumn{3}{c}{Skeleton without error} & \multicolumn{3}{c}{Skeleton with 20\% of error} \\
      & AUPR & SHD & SID & AUPR & SHD & SID  \\
    \midrule
    \textit{Constraints}\\
    PC-Gauss & 0.67 (0.11) & 9.0 (3.4) & 131 (70) & 0.42 (0.06) & 21.8 (5.5) & 191.3 (73)  \\
    PC-HSIC & 0.80 (0.08) & 6.7 (3.2) & 80.1 (38) & 0.49 (0.06) & 19.8 (5.1) & 165.1 (67)  \\
    \midrule
    \textit{Pairwise}\\
    ANM & 0.67 (0.11) & 7.5 (3.0) & 135.4 (63) & 0.52 (0.10) & 19.2 (5.5) & 171.6  (66)  \\
    Jarfo & 0.74 (0.10)   & 8.1 (4.7) & 147.1 (94) & 0.58 (0.09) & 20.0 (6.8) & 184.8 (88) \\
    \midrule
    \textit{Score-based}\\
        GES & 0.48 (0.13) & 14.1 (5.8) & 186.4 (86) & 0.37 (0.08) & 20.9 (5.5) & 209 (83) \\ 
     LiNGAM & 0.65 (0.10) & 9.6 (3.8) & 171 (86) & 0.53 (0.10) & 20.9 (6.8) & 196 (83)  \\
     CAM & 0.69 (0.13)  & 7.0 (4.3) & 122 (76) & 0.51 (0.11) & \underline{15.6} (5.7) & 175 (80) \\    
    \textbf{CGNN} ($\widehat{\text{MMD}}^m_k$) & 0.77 (0.09) &  7.1 (2.7) &  141 (59) & 0.54 (0.08) & 20 (10) & 179 (102) \\
    \textbf{CGNN} ($\widehat{\text{MMD}}_k$) & \underline{0.89}* (0.09) & \underline{2.5}* (2.0) & \underline{50.45}* (45) & \underline{0.62} (0.12) & 16.9 (4.5) & \underline{134.0}* (55) \\
    \bottomrule
  \end{tabular}
   \hspace*{-2cm}
  \label{table:multi}
\end{table*}

\begin{table*}[h!]
  \caption{Average (std. dev.) results for the orientation of 20 and 50 artificial graphs coming from Syntren simulator given true skeleton. $^*$ denotes statistical significance at $p=10^{-2}$. Underline values correspond to best scores.}
  \footnotesize
  \hspace*{-1.5cm}
    \centering
  \begin{tabular}{l|ccc|ccc|}
    \toprule
     &  \multicolumn{3}{c}{Syntren network 20 nodes} &  \multicolumn{3}{c}{Syntren network 50 nodes} \\
     & AUPR & SHD & SID & AUPR & SHD & SID  \\
    \midrule
    \textit{Constraints}\\
    PC-Gauss & 0.40 (0.16) & 16.3 (3.1) & 198 (57) & 0.22 (0.03) & 61.5 (32) & 993 (546) \\
    PC-HSIC & 0.38 (0.15) & 23 (1.7) & 175 (16) & - & - & -  \\
    \midrule
    \textit{Pairwise}\\
    ANM & 0.36 (0.17) & 10.1 (4.2) & 138 (56) & 0.35 (0.12) & 29.8 (13.5) & 677 (313) \\
    Jarfo & 0.42 (0.17)  & 10.5 (2.6) & 148 (64) & 0.45 (0.13) & 26.2 (14) & 610 (355)\\
    \midrule
    \textit{Score-based}\\
        GES & 0.44 (0.17) & 9.8 (5.0) & 116 (64) & 0.52 (0.03) & 21 (11) & 462 (248)\\ 
     LiNGAM & 0.40 (0.22) & 10.1 (4.4) & 135 (57) & 0.37 (0.28) & 33.4 (19) & 757 (433)   \\
     CAM & 0.73 (0.08)  & 4.0 (2.5) & 49 (24) & 0.69 (0.05) & 14.8 (7) & 285 (136) \\    
    \textbf{CGNN} ($\widehat{\text{MMD}}^m_k$) & \underline{0.80}* (0.12) &  3.2 (1.6) & 45 (25) & \underline{0.82}* (0.1) & \underline{10.2}* (5.3) & \underline{247} (134)  \\
    \textbf{CGNN} ($\widehat{\text{MMD}}_k$) & 0.79 (0.12) & \underline{3.1}* (2.2) & \underline{43} (26) & 0.75 (0.09) & 12.2 (5.5) & 309 (140) \\
    \bottomrule
  \end{tabular}
   \hspace*{-2cm}
  \label{table:syntren_network}
\end{table*}

\begin{table*}[ht!]
  \caption{Average (std. dev.) results for the orientation of the real protein network given true skeleton. $^*$ denotes statistical significance at $p=10^{-2}$. Underline values correspond to best scores.}
  \footnotesize
  \hspace*{-1.5cm}
    \centering
  \begin{tabular}{l|ccc|}
    \toprule
     &  \multicolumn{3}{c}{Causal protein network}\\
     & AUPR & SHD & SID \\
    \midrule
    \textit{Constraints}\\
    PC-Gauss & 0.19 (0.07) & 16.4 (1.3) & 91.9 (12.3) \\
  
    PC-HSIC & 0.18 (0.01) & 17.1 (1.1) & 90.8 (2.6) \\
    \midrule
    \textit{Pairwise}\\
    ANM & 0.34 (0.05) & 8.6 (1.3) & 85.9 (10.1)\\
    Jarfo & 0.33 (0.02) & 10.2 (0.8) & 92.2 (5.2)\\
    \midrule
    \textit{Score-based}\\
        GES & 0.26 (0.01) & 12.1 (0.3) & 92.3 (5.4)\\ 
     LiNGAM & 0.29 (0.03) & 10.5 (0.8) & 83.1 (4.8) \\
     CAM & 0.37 (0.10) & 8.5 (2.2) & 78.1 (10.3)\\    
    \textbf{CGNN} ($\widehat{\text{MMD}}^m_k$) & 0.68 (0.07) & 5.7 (1.7) & 56.6 (10.0) \\
    \textbf{CGNN} ($\widehat{\text{MMD}}_k$) & \underline{0.74}* (0.09)  & \underline{4.3}* (1.6) & \underline{46.6}* (12.4)\\
    \bottomrule
  \end{tabular}
   \hspace*{-2cm}
  \label{table:cyto_network}
\end{table*}

\end{document}